\documentclass[twoside]{article}

\usepackage[utf8]{inputenc} 
\usepackage[T1]{fontenc}    
\usepackage{hyperref}       
\usepackage{url}            
\usepackage{booktabs}       
\usepackage{amsfonts}       
\usepackage{nicefrac}       
\usepackage{microtype}      
\usepackage{xcolor}         

\usepackage[american]{babel}

\usepackage{mathtools} 
\usepackage{booktabs} 
\usepackage{tikz} 

\usepackage{algorithm}
\usepackage{algorithmic}

\usepackage{microtype}
\usepackage{graphicx}
\usepackage{booktabs}

\usepackage{amsthm}
\usepackage{thmtools,thm-restate}

\usepackage{hyperref}
\usepackage{caption}
\usepackage{subcaption}
\usepackage{bbm}

\usepackage{amsmath}
\usepackage{amssymb}
\usepackage{mathtools}

\usepackage{mwe}
\usepackage{tikz} 
\usepackage{pgf}
\usepackage{bbm}
\usepackage{extarrows} 

\usepackage[capitalize,noabbrev]{cleveref}

\usepackage{enumitem}

\newcommand{\A}[0]{\mathcal{A}}
\newcommand{\B}[0]{\mathcal{B}}
\newcommand{\V}[0]{\mathcal{V}}
\newcommand{\scal}[0]{\mathcal{S}}
\newcommand{\ocal}[0]{\mathcal{O}}
\newcommand{\pcal}[0]{\mathcal{P}}
\newcommand{\E}[0]{\mathcal{E}}
\newcommand{\F}[0]{\mathcal{F}}
\newcommand{\I}[0]{\mathcal{I}}
\newcommand{\pr}[0]{\mathbb{P}}

\newcommand{\G}[0]{\mathcal{G}}

\newcommand{\X}[0]{\mathcal{X}}
\newcommand{\N}[0]{\mathcal{N}}
\newcommand{\M}[0]{\mathcal{M}}
\newcommand{\J}[0]{\mathcal{J}}

\newcommand{\qb}[0]{\mathbf{q}}
\newcommand{\bb}[0]{\mathbf{b}}
\newcommand{\Rb}[0]{\mathbf{R}}

\newcommand{\pa}[1]{\text{Pa}_{\G}\paran{#1}} 
\newcommand{\pay}{\text{Pa}_Y}
\newcommand{\vps}[0]{\V_{p,s}}
\newcommand{\vpzs}[0]{\V_{p_0,s}}
\newcommand{\tps}[0]{T_{p,s}}
\newcommand{\wps}[0]{w_{p,s}}

\newcommand{\paran}[1]{\left( #1 \right)}

\newcommand{\kl}[1]{\mathrm{KL}\left( #1 \right)}

\newcommand{\muh}[0]{\hat{\mu}}

\newcommand{\rione}{R_i^{(1)}}
\newcommand{\ritwo}{R_i^{(2)}}
\newcommand{\cucb}{C_{\text{UCB}}}
\newcommand{\delti}{\Delta T_i}
\newcommand{\ceillnm}{\left\lceil \frac{\ell n}{m} \right\rceil}

\theoremstyle{plain}
\newtheorem{theorem}{Theorem}[section]
\newtheorem{proposition}[theorem]{Proposition}
\newtheorem{lemma}[theorem]{Lemma}

\theoremstyle{definition}
\newtheorem{definition}[theorem]{Definition}

\theoremstyle{remark}
\newtheorem{remark}[theorem]{Remark}

\usepackage[textsize=tiny]{todonotes}
\usepackage[accepted]{aistats2026}
\begin{document}

%

%

\twocolumn[

\aistatstitle{Graph Learning Is Suboptimal in Causal Bandits
}

\aistatsauthor{ Mohammad Shahverdikondori \And Jalal Etesami \And  Negar Kiyavash }

\aistatsaddress{ College of Management \\ of Technology, EPFL \And  Department of Computer \\ Science, TU Munich \And College of Management \\ of Technology, EPFL } ]

\begin{abstract}
  We study regret minimization in causal bandits under causal sufficiency where the underlying causal structure is not known to the agent. Previous work has focused on identifying the reward’s parents and then applying classic bandit methods to them, or jointly learning the parents while minimizing regret.  We investigate whether such strategies are optimal. Somewhat counterintuitively, our results show that learning the parent set is suboptimal. We do so by proving that there exist instances where regret minimization and parent identification are fundamentally conflicting objectives. We further analyze both the known and unknown parent set size regimes, establish novel regret lower bounds that capture the combinatorial structure of the action space. Building on these insights, we propose nearly optimal algorithms that bypass graph and parent recovery, demonstrating that parent identification is indeed unnecessary for regret minimization. Experiments confirm that there exists a large performance gap between our method and existing baselines in various environments.
\end{abstract}


\section{INTRODUCTION}
Multi-armed bandit (MAB) settings are a fundamental framework for modeling sequential decision-making in stochastic environments \cite{bandit-book2-bubeck2012regret, bandit-book1-lattimore2020bandit}, with applications ranging from recommendation systems \cite{mab5-recom-li2010contextual, mab6-recom-bouneffouf2019survey,sankagiri2025recommendations}, clinical trials \cite{mab3-clinical-villar2015multi, mab4-clinical-liu2020reinforcement}, to A/B tests \cite{mab7-abtesting-scott2010modern}. In this framework, an agent repeatedly selects actions (arms), observes the resulting rewards, and aims to maximize the expected cumulative reward. In the classic MAB problem, arms are assumed to be independent. This limits its applicability in more structured environments. To remedy this limitation, various extensions have introduced dependencies among arms, including linear bandits \cite{mab-linear1-dani2008stochastic, mab-linear2-abbasi2011improved}, bandits with graph feedback \cite{graph-feedback1-alon2015online, graph=feedback2-alon2017nonstochastic}, bandits with interference \cite{interference1-jia2024multi, interference2-agarwal2024mutli, interference3-jamshidi2025graph}, and causal bandits \cite{lattimore2016causal, yabe2018causal, pomis-lee2018structural}. 

In causal bandits, the focus of this paper, dependencies are captured by a causal graph over the variables, with one node designated as the reward. Actions correspond to interventions on subsets of variables, after which the agent observes the reward along with the values of non-intervened variables. Exploiting this causal structure and the additional observations has been shown to significantly accelerate learning \cite{lattimore2016causal, CUCB-lu2020regret}.

A common assumption in much of the causal bandit literature is knowledge of the full causal graph. Although progress has been made in developing scalable causal discovery methods \cite{discovery1-glymour2019review, discovery2-vowels2022d, discovery3-lam2022greedy, discovery4-shahverdikondori2024qwo, discovery5-mokhtarian2025recursive}, such methods typically require full knowledge of observational or interventional distributions and remain imperfect. Therefore, in practice it is more likely that the structure of the causal graph is unknown. Under causal sufficiency (i.e., no unobserved variables) and for agents that can intervene on subsets of nodes, it was shown that the optimal action is an intervention on all parents of the reward node \cite{pomis-lee2018structural}. Consequently, prior work has focused on identifying the parent set and then applying standard regret-minimization algorithms to it, or attempting to simultaneously learn the parents and minimize regret \cite{unknwon-tree-lu2021causal, mikhail-konobeev2025causal, linear4-unknown-peng2025asymmetric, elahi2024partial}.

In this work, we ask whether identifying the parent set is indeed optimal for regret minimization. We show that, with no assumptions on the underlying generating model and for the worst-case regret minimization, even when the graph over non-reward variables is fully known, parent identification is sub-optimal. 

It is important to note that our analysis does not impose any structural assumptions (e.g., linearity) on the causal model and focuses on hard interventions, where the agent sets fixed values for a subset of variables. Our results are specific to this setting. In contrast, some related works consider other classes of interventions, such as those that modify conditional distributions \cite{pure2-sen2017identifying,simoes2025minimal}, and our conclusions may not directly extend to those scenarios.

Our main contributions are as follows:
\begin{itemize}
    \item We demonstrate that parent identification and regret minimization can be fundamentally at odds. That is, there exist instances where the set of high-reward actions is disjoint from the set of informative actions to identify the parents.
    \item When the number of parents of the reward node is known, we establish a worst-case regret lower bound that reflects the combinatorial structure of the action space and holds even if the agent has complete knowledge about the graph over the non-reward variables. We further propose a simple algorithm that, under a mild assumption on the problem parameters, achieves regret matching this lower bound up to logarithmic factors, without requiring any prior knowledge.
    \item When the number of parents is unknown, we prove a lower bound, showing that no algorithm can attain the same regret rates as in the known case uniformly over all parent set sizes. In addition, we introduce an adaptive algorithm that adjusts to the unknown parent size. This algorithm is Pareto optimal, up to logarithmic factors, when the agent can intervene on all variables in each round, and is nearly Pareto optimal in more general settings.
    \item Our experiments in diverse environments show that our algorithms outperform the existing baselines and reduce regret by up to a factor of $20$.
\end{itemize}

\subsection{Related Work}

\paragraph{Causal Bandits.} 
The causal bandit problem was first introduced in \cite{lattimore2016causal}, where the authors considered simple regret minimization with access to causal background knowledge, showing that such knowledge can significantly improve learning efficiency. Since then, causal bandits have been studied under a variety of settings and assumptions. Examples include assuming access to the distribution of parents of the reward node under each intervention~\cite{CUCB-lu2020regret, pre-pareto-bilodeau2022adaptively, pareto-liucausal, pure2-sen2017identifying}, restricting the causal model to be linear~\cite{linear1-varici2023causal, linear2-yan2024robust, linear3-yan2024linear}, incorporating action costs into a budget constraint~\cite{budget-nair2021budgeted, jamshidi2024confounded}, studying combinatorial causal bandits with binary generalized linear models (BGLM)~\cite{feng2023combinatorial, pure-xiongcombinatorial}, and addressing best-arm identification~\cite{shahverdikondori2025optimal, unknown-comb-feng2025combinatorial}. Another line of work proposed offline approaches to reduce the action space by identifying so-called possibly optimal intervention sets (POMIS) before the learning process begins~\cite{pomis-lee2018structural, non-manipulable-lee2019structural}.

Causal bandits with an \emph{unknown graph structure} have also been investigated. In~\cite{unknwon-tree-lu2021causal, mikhail-konobeev2025causal}, the authors proposed algorithms in the atomic intervention setting that first identify the parents of the reward node and later run a standard regret minimization algorithm on the identified set. These algorithms are inherently limited to atomic interventions. \cite{unknwon-tree-lu2021causal} proved a lower bound that showed when the reward node has a single parent and only atomic interventions are allowed, without additional distributional assumptions, no algorithm can achieve regret better than that of standard bandit algorithms applied to the full action set. The combinatorial causal bandits with an unknown graph under the BGLM model was studied in \cite{unknown-comb-feng2025combinatorial}.
\cite{additive-malek2023additive} studied causal bandits with an unknown graph under an additive outcome assumption and general interventions, showing that the problem can be reduced to an additive combinatorial linear bandit with full-bandit feedback.
For continuous models, \cite{linear3-yan2024linear} studied unknown-graph causal bandits under linear structural equations. More recently, \cite{linear4-unknown-peng2025asymmetric} considered unknown-graph causal bandits with soft interventions under a linear structural equation model and proposed a sub-graph learning UCB method that controls false negative graph errors. Similarly, \cite{unknown-new2-zhao2025causal} studied unknown-graph causal bandits under a Gaussian linear DAG, using backdoor adjustment to combine observational and experimental data in a UCB algorithm. Finally, \cite{elahi2024partial} demonstrated that, when the graph is unknown, partial causal discovery suffices to identify the set of POMISs.

In contrast to the aforementioned works, we do not assume any particular structure on the graph or the distribution. We allow interventions of size $m$, thus generalizing previous studies that only consider atomic or general interventions.

\paragraph{Adaptivity to Unknown Parameters.} 
Since we study the case where the size of the reward's parent set is unknown and provide an algorithm that adapts to it, our work is also related to the broader literature on adaptivity to unknown parameters in bandit problems. The closest work is \cite{multiple-zhu2020regret}, which considers bandits with multiple optimal arms where the number of optimal arms is unknown to the learner. Other examples include continuum-armed bandits~\cite{x1-agrawal1995continuum, x2-locatelli2018adaptivity, x3-hadiji2019polynomial}, where algorithms are designed to adapt to an unknown smoothness parameter.

\section{PRELIMINARIES AND PROBLEM SETUP}

In this section, we formally introduce the problem setup for the causal bandits problem with an unknown graph and fixed-sized interventions, adapting a terminology similar to prior work.

A causal graph $\G$ over $n$ random variables $\X = \{X_1, X_2, \ldots, X_n\}$ is a directed acyclic graph (DAG), where an edge $X_i \rightarrow X_j$ indicates that $X_i$ directly causes $X_j$, i.e., changes in the value of $X_i$ while keeping all other variables fixed may alter the distribution of $X_j$. 
We assume causal sufficiency holds, i.e., there are no unobserved variables.  
Let $\pa{X_i}$ denote the set of parents of node $X_i$ in $\G$. We assume that each variable in $\G$ takes values in $[\ell] \coloneq \{1, \ldots, \ell\}$ for some positive integer $\ell$.\footnote{The extension to finite domains with different cardinalities is straightforward.}

The causal bandits problem is a sequential game between an agent and an environment. The environment selects a causal graph $\G$ on variables $X_1, \ldots, X_n$ and a reward variable $Y$, together with the conditional dependencies consistent with the graph. Let $\pay \subseteq \X$ denote the set of parents of the reward node in the causal graph with cardinality $k \coloneq  |\pay|$. 
We assume that the agent does not know $\pay$ but may have knowledge of the graph $\G$ over $\X$ or the value of $k$.

\textbf{Action:} At each round, the agent selects a subset of variables of size at most $m$ and fixes their values through intervention. 
Such a selection is called an \emph{action} or playing an \emph{arm}. We use these terms interchangeably.
Let $P([n])$ denote the power set of $[n]$. Each action $a$ is defined by a pair $(p_a, s_a)$, where $p_a \in P([n])$ is a subset of $[n]$ indicating the indices of the intervened variables with $|p_{a}| \leq m$, and $s_a\in [\ell]^{|p_a|}$ is the corresponding vector of their assigned values.  
Let $\mathcal{A}$ denote the set of all such possible actions. 
We also define $\A_i \coloneq  \{ a \in \A \mid |p_a| = i \}$ as the set of actions with intervention size $i$. Then, $|\A_i| = \binom{n}{i} \ell^i$ and $|\mathcal{A}|=\sum_{i=0}^m \binom{n}{i} \ell^i$. 
We assume that $p_a$ is sorted in increasing order and that $s_a$ is aligned with this ordering: the $i$-th entry of $s_a$ specifies the value assigned to the variable whose index is the $i$-th element of $p_a$.
Formally, at round $t$, the agent chooses an action 
$
a_t = (p_{a_t}, s_{a_t})\in \mathcal{A}.
$
With a slight abuse of notation, for any set $S$ and integer $r$, we denote by $\binom{S}{r}$, the family of all subsets of $S$ with $r$ elements. 


\textbf{Reward:} After performing action $a_t$ at round $t$, the agent observes a vector $\paran{\mathbf{x}^{(t)}, y_t} = (x_{1,t}, \ldots, x_{n,t}, y_t)$, which is a sample from the post-interventional distribution $\pr \paran{ X_1, \ldots, X_n, Y \mid do(\mathbf{X}_{p_{a_{t}}} = s_{a_t})}$. 
We denote the expected reward of this action by $\mu_{a_t} \coloneq  \mathbb{E}[Y \mid a_t]=\mathbb{E}[Y \mid do(\mathbf{X}_{p_{a_{t}}} = s_{a_t})]$.
Let $a^* \in \arg\max_{a \in \A} \mu_a$ be an optimal action. We assume that $\mu_a \in [0,1]$ for all $a \in \A$ and that the distribution of $Y$ under each action is $1$-sub-Gaussian, which is a common assumption in the bandit literature \cite{bandit-book2-bubeck2012regret, bandit-book1-lattimore2020bandit}.

\textbf{Environments:} We denote by $\E(n, \ell, k)$ the class of environments with $n$ variables, each taking values in $[\ell]$, a reward node with $k$ parents and 1-sub-Gaussian reward distributions with means in $[0,1]$.
When the parameters are fixed, we write $\E$.

\textbf{Regret:} 
Agent's policy, denoted by $\pi$, represents a sequence of actions over time  i.e., $\pi=(a_1, a_2,...)$. 
Accordingly,  the cumulative regret of a policy $\pi$ in an environment $\V \in \E$ over $T$ rounds is defined as
$$
R_T(\pi, \V) \coloneq  T \mu_{a^*} - \sum_{t \in [T]} \mu_{a_t}.
$$
This corresponds to the cumulative gap between the expected reward of the optimal action and that of the actions chosen by $\pi$. 

\textbf{Agent's objective:} 
To encode the different assumptions about the agent’s prior knowledge, we introduce an \textit{information set} $\I$ that specifies what the agent knows in advance $\big($e.g., $\I = \emptyset$, $\{k\}$, $\{\G\}$, or $\{k,\G\} \big)$ which correspond to the agent having no side information, knowing the number of parents of the reward node,  the causal graph over the random variables $\X$, or both, respectively).  
We define $\Pi(\I)$ as the set of all policies that can utilize prior information $\I$. 
The agent's goal is to design a policy $\pi\in\Pi(\I)$ that minimizes the worst-case cumulative regret over the set of possible instances defined by
$$
R_T(\pi, \E) \coloneq  \max_{\V \in \E} R_T(\pi, \V).
$$
Next lemma shows that there always exists an optimal action in $\A_m$. This allows us to design algorithms that only explore the actions in $\A_m$.

\begin{restatable}{lemma}{optimalActionAm} \label{lem: optimal action A_m}
    Under causal sufficiency, for any values of $n,\ell,k,m$ and any instance $\V \in \E$, there exists an optimal action with the maximum intervention size, that is
$
        \max_{a \in \A_m} \mu_a \;=\; \mu_{a^*}.
$
\end{restatable}

\section{GRAPH LEARNING MIGHT BE SUBOPTIMAL} \label{sec: trade-off}

Under causal sufficiency and for action sizes at least as large as the reward node’s parent set ($m \geq k$), prior work~\cite{unknwon-tree-lu2021causal, mikhail-konobeev2025causal} establishes that the optimal action corresponds to an intervention on the parent set.
In the presence of unobserved variables, the parent set is replaced by the set of possibly optimal minimal intervention sets (POMIS)~\cite{pomis-lee2018structural}.
Thus, in causal bandits, if the causal graph is known, this knowledge can be used to reduce the exploration set. However, when the causal graph is unknown, the following fundamental question arises: \emph{does regret minimization necessarily require identification of the parent set?}

In this section, we show that, under causal sufficiency and without any distributional assumptions, identifying the parent set is not necessarily optimal for minimizing cumulative regret. In fact, we shall see that there are instances where the two objectives, regret minimization and parent set identification, are at odds. Thus, in such cases, no algorithm that achieves the optimal regret rate can simultaneously identify the parent set with high probability. To formalize this, we define the probability of parent misidentification for a given algorithm.

\begin{definition}[Parent-Identification] \label{def: parent-id}
In an environment $\V\in\E$, a parent identification algorithm uses the interventions generated by a policy $\pi\in\Pi(\I)$ for $T$ rounds along with a decision rule $\bar{d}$ to output an estimated parent set of the reward node, denoted by $\widehat{Pa}_T(\pi,\bar{d}, \V)\subseteq \X$. 
For a subclass $\E_0 \subseteq \E$ of instances, we define 
    $\delta_T \paran{\pi,\bar{d}, \E_0}$ as the maximum probability of misidentifying the true parent set after $T$ rounds over instances in $\E_0$, i.e.,
    $$
    \delta_T \paran{\pi, \bar{d} ,\E_0} \coloneq  \max_{\V \in \E_0} \pr \paran{\widehat{Pa}_T(\pi,\bar{d}, \V) \neq \pay(\V)}.
    $$ 
\end{definition}
To demonstrate the trade-off between the probability of identifying a reward's parent set and regret, we introduce a subclass of instances $\E_0$, in which any parent identification algorithm with $m = k$ that achieves low error rates (in terms of $T$) necessarily incurs suboptimal regret. To show that $\E_0$ is not a collection of degenerate instances where, for example, identifying the parent set is impossible, we propose a simple uniform sampling algorithm and prove that it achieves good performance in the parent identification task over $\E_0$. 

\textbf{Uniform Sampling.} 
Consider the case $m = k$. The uniform sampling policy $\pi_{\text{Unif}}$, on any instance $\V$, plays each  action $a \in \A_m$ equally often, that is, $\tfrac{T}{\binom{n}{m} v^m}$ times. 

\begin{restatable} [Identification-Regret Trade-Off] {theorem} {tradeOff} \label{them: id-regret-trade-off}
    There exists a subclass $\E_0 \subseteq \E(n,\ell,k)$ such that for $m = k$, the two following statements hold: 
    \begin{enumerate}[noitemsep,nolistsep]
        \item There exists a decision rule $\bar{d}_{\text{Unif}}$ that combined with the uniform sampling policy achieves $\delta_T \paran{\pi_{\text{Unif}},\bar{d}_{\text{Unif}}, \E_0} \in \ocal \paran{ \exp \paran{- T}}$. \vspace{0.2cm}
        \item The regret of any policy $\pi$ for which there exists a decision rule $\bar{d}$ such that $\delta_T \paran{\pi, \bar{d}, \E_0} \in \ocal \paran {\exp \paran{- T^{\alpha}}}$ grows as $R_T\big( \pi, \E_0, \{k\} \big) \in \Omega \paran{T ^ \alpha}$.
    \end{enumerate}
\end{restatable}

\textbf{Proof Sketch:} The key step in the proof of this result is the construction of a class of instances in which the set of high-reward actions and the set of informative actions for identifying the set of parents are disjoint. Consequently, at each round the learner must decide to either minimize the regret by playing an optimal action or to play an action that is informative for learning the parents. 
To illustrate the construction, consider an instance with $n$ binary variables and let $m = k$. The causal model is defined such that every variable is equal to $0$ unless intervened upon, and the first $k$ variables are the parents of all others. If all of these $k$ variables are set to $1$, then all remaining variables also become $1$ deterministically. The reward is defined to have mean $0$ for all actions except the one that sets the first $k$ variables to $1$, which yields the expected reward of $1$.

In this setting, the unique high-reward action sets the first $k$ variables to $1$, but this simultaneously forces all other variables to be $1$, making it impossible to distinguish which variables actually influence the reward. To identify the parent set, the learner must instead intervene on arbitrary subsets of $k$ variables and test whether setting them to $1$ increases the reward. However, such actions are suboptimal and incur regret. This complete separation between high-reward and informative actions is the fundamental source of the trade-off.

The construction used in the proof is intentionally degenerate to highlight this phenomenon clearly. While similar lower bounds might be established for non-degenerate models, doing so requires more involved constructions, and the trade-off might be weaker. For instance, suppose we impose a non-degeneracy condition such that each variable takes any possible value with probability at least $\epsilon > 0$ under any parent configuration. In the regime $T \gg 1/\epsilon$, the situation becomes more nuanced. Even when the learner plays the optimal (low-regret) action, it observes alternative configurations of other variables with non-zero probability. These observations can reveal that changes in those variables do not affect the reward, effectively providing information about the parent set ``for free.'' This suggests that, in such settings, the trade-off would be weaker.
More details are in Appendix \ref{proof:id-regret-trade-off}.

Theorem~\ref{them: id-regret-trade-off} shows the fundamental trade-off between the cumulative regret minimization and the parent identification. More precisely, it implies that any policy achieving the same parent-identification performance as the one using $\bar{d}_{\text{Unif}}$ and uniform sampling (i.e., $\alpha = 1$) necessarily suffers linear cumulative regret. 
It is noteworthy that there are other pure-exploration objectives in bandit literature without any fundamental trade-off with the cumulative regret. For instance, \emph{simple regret} objective seeks to minimize the expected difference between the estimated and true best arm~\cite{simple1-bubeck2009pure, simple2-zhao2023revisiting}. For simple regret, it is established that there is no fundamental trade-off with cumulative regret: algorithms that are optimal for cumulative regret can, up to constant factors, also achieve optimal simple regret. For further details, see \cite[Section~33]{bandit-book1-lattimore2020bandit}.


\section{KNOWN PARENT SIZE} \label{sec: known-k}

In this section, we address the cumulative regret minimization problem in causal bandits with an unknown graph, where the number of parents of the reward node $k$ is known to the agent, i.e., $k \in \I$. However, the exact parent set $\pay$ is not known. 
We consider two regimes based on the relation between the intervention size ($m$) and the number of parents ($k$): $m \geq k$ and $m < k$. 
First, we establish lower bounds on the worst-case regret for both regimes, and then propose an algorithm together with an upper bound on its regret that is close to the lower bounds in both regimes.

\subsection{Lower Bounds} 


The results of this section highlight the difficulty of worst-case regret minimization in causal bandits when the agent knows only the number of parents $k$ but not their identities. The bounds are information-theoretic and hold for any policy $\pi$.

An important feature of both lower bounds is that they remain valid even under the additional assumption that the agent has full knowledge of $\G$, the causal graph over $\X$, the set of non-reward variables. Moreover, the bounds hold for any such causal graph. Therefore, the crucial information to minimize the regret is the knowledge of the parent set $\pay$ as knowing  $\G$ does not improve the worst-case regret guarantees.

\begin{restatable} [$m \geq k$ Lower Bound] {theorem} {mgeqkLB} \label{thm: m-geq-k-lower-bound}
    For any policy $\pi\in\Pi(\{k, \G\})$, any values $n \geq m \geq k$, and any causal graph $\G$:
        \begin{align} \label{eq: m-geq-k-lower-bound}
        R_T \big( \pi, \E \big) \in \Omega \paran{\sqrt{T \max \paran{(\ell-1)^k \frac{\binom{n}{k}}{\binom{m}{k}}, \; \ell^k}}}.
    \end{align}
\end{restatable}

\textbf{Proof Sketch:} To prove this theorem, we establish a novel and more general statement. 
Let $\V_0 \in \E$ denote the neutral instance where the graph $\G$ is the empty graph, meaning that the variables are random and independent such that the value of any non-intervened variable is always $1$, and the reward distribution is always $\mathcal{N}(0,1)$ regardless of the values of its parents. 
Fix a policy $\pi\in\Pi(\{k, \G \} )$. For any pair $(p,s)$ with $p$ denoting a subset of $[n]$ of size $k$ and $s \in [\ell]^k$, let $w_{p,s}(\pi)$ denote the expected fraction of times $t$ in $T$ rounds of interaction between $\pi$ and $\V_0$ such that $\mathbf{x}^{(t)}_p = s$. Let $\pcal_{k}$ denote the set of all such pairs $(p,s)$. We prove the following lower bound, which holds for  any graph $\G$:
$$
R_T \big(\pi, \E\big) \in \Omega \!\left( 
\max_{\J \subseteq \pcal_{k}} (1 - c_{\J})^2
\sqrt{\,T \; \frac{\left| \J \right|}{\sum_{(p,s) \in \J} w_{p,s}(\pi)} }
\right),
$$
where $c_{\J}$ is the maximum value of the fraction $\frac{\sum_{(p,s) \in \J} w_{p,s}(\pi)}{\left| \J \right|}$ over all policies. 
We then focus on a collection of candidate sets $\scal$ with $|\scal| = k+1$, and show that maximizing the expression over $\J \in \scal$ yields the regret lower bound stated in the theorem.

In the second regime,  $m < k$, the agent cannot intervene on all parents simultaneously. The bound is obtained by considering instances in which there are $k$ true parents, but only $m$ of them affect the reward. 


\begin{restatable} [$m < k$ Lower Bound] {theorem} {mlesskLB} \label{thm: m-less-k-lower-bound}
    For any policy $\pi\in\Pi(\{k, \G\})$, any values $n \geq k > m$, and any graph $\G$, we have 
    \begin{align} \label{eq: m-less-k-lower-bound}
        R_T \big( \pi, \E \big) \in \Omega \paran{\sqrt{T \max \paran{(\ell-1)^m {\binom{n}{m}}, \; \ell^m}}}.
    \end{align}
\end{restatable}

\begin{remark} [Interpretation of the Lower Bounds] \label{rem: known-k-lower-bound}
    To interpret the bounds more concretely, consider the case where $\ell \in \Omega(k)$. This is by assuming that the number of reward's parents in the causal graph is small. 
    In this case,  $\frac{\ell^k}{(\ell-1)^k} \in \ocal(1)$ and for any $\pi\in\Pi(\{k, \G\})$, the following holds:
    \begin{align}\label{eq:lower:remark0}
        &R_T \big( \pi, \E \big) \in \Omega \paran{\sqrt{T \, \ell^k \frac{\binom{n}{k}}{\binom{m}{k}}}}, \quad \text{for } m \geq k,\\ \label{eq:lower:remark1}
        &    R_T \big( \pi, \E\big) \in \Omega \paran{\sqrt{T \, \ell^m \binom{n}{m}}}, \quad \text{for } m < k.
    \end{align}    
Note that when $m < k$, the quantity $\ell^m \binom{n}{m}$ is exactly the number of actions in $\A_m$, i.e., the set of all interventions on $m$ variables. Therefore, no algorithm can achieve a better worst-case performance than running a standard UCB algorithm over the  action set $\A_m$.
\end{remark}

\subsection{Algorithm and Upper Bound} \label{sec:alg_known}

We are now ready to propose a simple algorithm that achieves near-optimal performance on broad classes of instances. 
The pseudo-code is presented in Algorithm~\ref{algo: known-k}.
Note that this algorithm does not attempt to identify the parent set $\pay$ or to recover any other causal relation among the variables. Moreover, it does not require knowledge of the causal graph $\G$; its information set is $\I=\{k\}$, i.e., the number of reward's parents. The algorithm is solely designed for regret minimization and may find the optimal action without explicitly inferring which variables are reward's parents.

For the case $m \geq k$, 
every action that intervenes on the set $\pay$ and assigns it the optimal values achieves the same expected reward, since the reward distribution depends only on the values of the parents. Leveraging this observation, Algorithm~\ref{algo: known-k} samples a random subset $\A' \subseteq \A_m$ of appropriate  size 
such that with high probability $\A'$ contains an optimal action. The algorithm then runs the standard UCB algorithm on the actions in $\A'$.
For the case $m < k$, the algorithm directly runs UCB on the entire action set $\A_m$.


\begin{algorithm}[h]
    \caption{}\label{algo: known-k}
    \begin{algorithmic}[1]      
        \STATE \textbf{Input.} The integers $m, n, k, \ell, T$. ($\G$ is unknown) 
        \IF{$m \geq k$}
            \STATE Set $n_0 = \min \paran{\ell^k \frac{\binom{n}{k}}{\binom{m}{k}} \ln \sqrt{T} \; , \; \ell^m \binom{n}{m}}$. 
        \ELSE
            \STATE Set $n_0 = \ell^m \binom{n}{m}$.
        \ENDIF
        \STATE Construct a uniformly random subset $\A' \subseteq \A_m$ with $|\A'| = n_0$. 
        \STATE Run the standard Upper Confidence Bound (UCB) algorithm on the arms in $\A'$ for $T$ rounds.
    \end{algorithmic}
\end{algorithm}

\begin{restatable} [Known $k$ Upper Bound] {theorem} {knownkUB} \label{thm: known-k-upper-bound}
    The worst-case regret of Algorithm \ref{algo: known-k} with input $k$ is  bounded by
    \begin{align} \label{eq: known-k-upper-bound}
        R_T \big( \texttt{Alg.1}[k], \E(n,\ell,k) \big)  
        \in \begin{cases}
            \tilde{\ocal} \paran{\sqrt{T  \ell^k \frac{\binom{n}{k}}{\binom{m}{k}}}}, \ m \geq k, \\
            \tilde{\ocal} \paran{\sqrt{T  \ell^m {\binom{n}{m}}}}, \ m < k,
        \end{cases}
    \end{align}
    where $\tilde{\ocal}$ hides constants and logarithmic factors.
\end{restatable}

Comparing the lower bounds for when $\ell \in \Omega(k)$ given in \eqref{eq:lower:remark0} and \eqref{eq:lower:remark1} with the regret bound of Algorithm~\ref{algo: known-k} in \eqref{eq: known-k-upper-bound} shows that Algorithm~\ref{algo: known-k} is optimal in both regimes up to logarithmic factors. 
Moreover, since this algorithm does not require knowledge of $\G$, whereas the lower bounds apply even to algorithms with access to $\G$,
we conclude that prior knowledge of $\G$ does not improve the worst-case regret bounds.

\begin{remark}
    In the case of $m = n$, the derived lower and upper bounds match (up to logarithmic factors), yielding a tight bound of $\tilde{\ocal} \paran{\sqrt{T \, \ell^k}}$. This is precisely the regret bound of the setting where the agent knows the exact set of $\pay$ and applies standard regret minimization algorithms over the $\ell^k$ possible arms. Hence, surprisingly, when the agent can intervene on all variables in each round, the regret with full knowledge of the graph (including the parents of the reward) is equal to that with no knowledge.   
\end{remark}

\section{UNKNOWN PARENT SIZE} \label{sec: unknown-k}

In this section, we drop the assumption that the agent knows $k$, the number of reward's parents. We first establish a lower bound showing that no algorithm can fully adapt to the unknown value of $k$ without incurring a penalty: specifically, it is impossible to achieve the same regret bounds as in the known-$k$ case uniformly over all values of $k$. 
Afterwards, we 
propose an algorithm with an empty information set, $\I=\{\}$, whose regret differs from our lower bound by only a small margin.

\paragraph{Change of performance measure:} 
When an algorithm does not know the value of $k$, its performance may vary significantly across instances with different $k$. 
Therefore, its performance should be assessed for all possible values of $k$, i.e., $k\in[n]$ rather than a single $k$. 
However, as indicated by the upper and lower bounds from the previous section, for $k > m$, any action could be optimal, meaning the algorithm must explore all actions regardless of $k$. Therefore, we focus on the regime $k \leq m \leq n$ and evaluate the performance of a policy $\pi$ using the following vector, which we call the \emph{regret vector} of the policy:
\begin{align*}
\Big[ R_T \big( \pi, \E(n,\ell,k)\big) \Big]_{k \in [m]}.    
\end{align*}
Recall that $R_T \big( \pi, \E(n,\ell,k)\big)$ denotes the worst-case regret of $\pi$ over all instances with $n$ variables and $k$ reward parents. 

One might consider the maximum worst-case regret over $k$ (i.e., the largest entry of the above vector) as a performance measure, but this is unsatisfactory: since for any $k_1 < k_2$, we have $\E(n,\ell,k_1) \subseteq \E(n,\ell,k_2)$, which implies
$$
R_T \big( \pi, \E(n,\ell,k_1)\big) \leq R_T \big( \pi, \E(n,\ell,k_2)\big).
$$
 Hence, the maximum entry of the regret vector always occurs at $k=m$. A trivial strategy that optimizes for $k=m$ is to run standard UCB on all actions; however, this is not necessarily the optimal approach. 
 On the other hand, the next theorem shows that any algorithm tailored to instances with $k = k_1$ necessarily exhibits sub-optimal performance on instances with $k > k_1$. 
 This highlights the need for algorithms that adapt to the unknown value of $k$ and perform well across all $k\in[m]$. 
To compare different algorithms, we use the notions of Pareto domination and Pareto optimality for their regret vectors.

\begin{definition}[Rate-Pareto Domination and Optimality] \label{def: pareto}
    Let $\mathbf{r}, \mathbf{s} \in \mathbb{R}_+^m$ be two regret vectors, where each entry is a function of the parameters $T, n, m, \ell$. We say that 
    $\mathbf{r}$ \emph{rate-Pareto dominates} $\mathbf{s}$ if there exists a universal constant 
    $C > 0$ such that
    $$
    \forall k \in [m]: \quad r_k \;\leq\; C \, s_k,
    $$ 
    and the reverse does not hold (i.e., there is no constant $C'>0$ such that 
    $\forall k \in [m]: \; s_k \leq C' r_k$).  
    In other words, $\mathbf{r}$ is not worse than $\mathbf{s}$ in every coordinate (up to constant factors) and is strictly better in at least one coordinate.  \\
    A regret vector $\mathbf{r}$ is said to be \emph{rate-Pareto optimal} for a set of policies 
    $\Pi$ if no regret vector corresponding to a policy in $\Pi$ rate-Pareto dominates $\mathbf{r}$.  
    A policy $\pi$ is \emph{rate-Pareto optimal} if its regret vector is rate-Pareto optimal.
\end{definition}



\subsection{Lower Bound}

Herein, we present a lower bound on the product of two distinct entries of the regret vector for any policy $\pi \in \Pi(\G)$.
  
\begin{restatable} [Unknown $k$ Lower Bound] {theorem} {unknownkLB} \label{thm: unknown-k-lower-bound}
    For any causal graph $\G$, any policy $\pi\in\Pi(\G)$, and any values $n \geq m \geq k_2 > k_1$, we have 
    \begin{align} \label{eq: unknown-k-lower-bound}
        & R_T \big( \pi, \E(n,\ell,k_1) \big)\times R_T \paran{\pi, \E(n,\ell,k_2)} \nonumber \\ 
        &\in \Omega \paran{{T \max \paran{(\ell-1)^{k_2} \frac{\binom{n - k_1}{k_2 - k_1}} {\binom{m - k_1}{k_2 - k_1}}, \; \ell^{k_2}}}}.
    \end{align}
\end{restatable}

This result shows a fundamental trade-off: for any policy in $\Pi(\{\G\})$, improving performance on instances with a fixed value of $k=k_1$ necessarily worsens performance on instances with a larger value of $k=k_2>k_1$. 

\subsection{Algorithm and Upper Bound}

As discussed earlier, if the goal were to minimize the maximum worst-case regret (i.e., minimizing the largest entry of the regret vector), Algorithm~\ref{algo: known-k} could be applied with the input $k=m$. In this case, we obtain the following for each $k \in [m]$,
$$
R_T \big( \texttt{Alg.1}[m], \E(n,v,k) \big)\! \in\! \ocal \!\left(\sqrt{T \, \ell^m \binom{n}{m}}\right)\!.
$$
However, based on the result of Theorem \ref{thm: unknown-k-lower-bound}, the incurred regret by Algorithm \ref{algo: known-k} is not optimal when the actual number of parents is less than $m$.  
We now propose an algorithm that adapts to the unknown value of $k$ and incurs lower regret when $k$ is small. The design is adapted from the setting of bandits with multiple optimal arms~\cite{multiple-zhu2020regret}, which in turn is inspired by the literature on continuum-armed bandits~\cite{x1-agrawal1995continuum, x2-locatelli2018adaptivity, x3-hadiji2019polynomial}. 

The algorithm proceeds in phases. In each phase $i$, it randomly selects a subset $S_i \subseteq \A_m$ of arms of size $q_i$, and then runs the standard UCB algorithm for $\Delta T_i$ rounds on the arms in $S_i$ together with the mixture arms constructed in earlier phases. At the end of phase $i$, the algorithm defines a \emph{mixture arm} $\tilde{a}_i$ based on the actions played during that phase. Formally, if the actions played in phase $i$ are $a_1,..., a_{\Delta T_i}$, then $\tilde{a}_i$ is the randomized arm that, when played, samples an action $a$ uniformly from $\{a_1,..., a_{\Delta T_i}\}$ and plays it. 
Intuitively, $\tilde{a}_i$ summarizes the exploration of phase $i$ into a single representative action, while preserving the empirical distribution of actions observed in that phase.

The schedule of the algorithm is such that $q_i$ is halved at each new phase, while the phase length $\Delta T_i$ is doubled. 
Pseudocode for this procedure is given in Algorithm~\ref{algo: unknown-k}, and the theorem below provides its regret upper bound on the class of instances with $k$ parents, for any $k \in [n]$.  It is important to note that Algorithm~\ref{algo: unknown-k} requires neither knowledge of $\G$ nor of $k$.

\begin{algorithm}[h]
    \caption{}\label{algo: unknown-k}
    \begin{algorithmic}[1]      
        \STATE \textbf{Input.} The integers $m, n, \ell, T$. ($\G$ is unknown)  
        \STATE \textbf{Initialization.} Set $i_f = \Big\lceil \log_2 \sqrt{T\frac{m}{\ell n}} \Big\rceil, \forall i \in [i_f]: q_i = 2^ {\lceil \log_2 \sqrt{T} \rceil - i + 1}, \Delta T_i = \ceillnm \, 2 ^{\lceil \log_2 \sqrt{T} \rceil + i}$, and $M = \emptyset$.
        \FOR{$i$ in $1,2,\ldots, i_f$}
            \STATE Construct $S_i \subseteq \A_m$ consisting if $q_i$ uniform random actions selected with replacement.
            \STATE Run UCB on actions in $S_i \cup M$ for $\Delta T_i$ rounds. 
            \STATE Construct the mixture arm $\tilde{a}_i$ from the UCB actions and add it to $M$. 
        \ENDFOR 
    \end{algorithmic}
\end{algorithm}

\begin{restatable} [Unknown $k$ Upper Bound] {theorem} {unknownkUB} \label{thm: unknown-k-upper-bound}
    For any $k$, the worst-case regret of Algorithm \ref{algo: unknown-k} on all the instances in $\E(n,\ell,k)$ is upper bounded as follows 
    \begin{align} \label{eq: unknown-k-upper-bound}
        R_T \paran{\texttt{Alg.2}, \E(n,\ell,k) }\!  
        \in\! \begin{cases}
            \tilde{\ocal} \paran{\sqrt{T  \frac{m}{n}}  \ell^{k- \frac12} \frac{\binom{n}{k}}{\binom{m}{k}}},  m \geq k, \\
            \tilde{\ocal} \paran{\sqrt{T  \frac{m}{n}}  \ell^{m - \frac12} {\binom{n}{m}}},  m < k,
        \end{cases}
    \end{align}
    where $\tilde{\ocal}$ hides constants and logarithmic factors.
\end{restatable}

Next lemma uses the results of Theorems \ref{thm: unknown-k-upper-bound} and \ref{thm: unknown-k-lower-bound} to show that Algorithm~\ref{algo: unknown-k} is near rate-Pareto optimal. 


\begin{restatable}[Pareto Optimality of Algorithm~\ref{algo: unknown-k}] {lemma}{pareto}\label{lem: pareto optimality}
The following statements hold:
\begin{itemize}
    \item When $m=n$, Algorithm~\ref{algo: unknown-k} is rate-Pareto optimal for $\Pi(\{\G\})$, up to logarithmic factors.
    
    \item In the general case, if $\ell \in \Omega(m)$, then the regret vector
    \begin{align*}
        &\Big[R_T \big(\texttt{Alg.2}, \E(n,\ell,1)\big), \;
        R_T \big(\texttt{Alg.2}, \E(n,\ell,2)\big)\tfrac{m}{n}, \\
        &\; \ldots, \;
        R_T \big(\texttt{Alg.2}, \E(n,\ell,m)\big)\tfrac{m}{n} \Big]
    \end{align*}
    is rate-Pareto optimal for $\Pi(\{\G\})$. Up to logarithmic terms, this vector coincides with the regret vector of Algorithm~\ref{algo: unknown-k} when $k=1$, and exhibits a multiplicative gap of $\tfrac{m}{n}$ for larger values of $k$.
\end{itemize}
\end{restatable}

\begin{figure*}[t]
    \centering
    \begin{subfigure}[t]{0.5\textwidth}
        \centering
        \includegraphics[width=\linewidth]{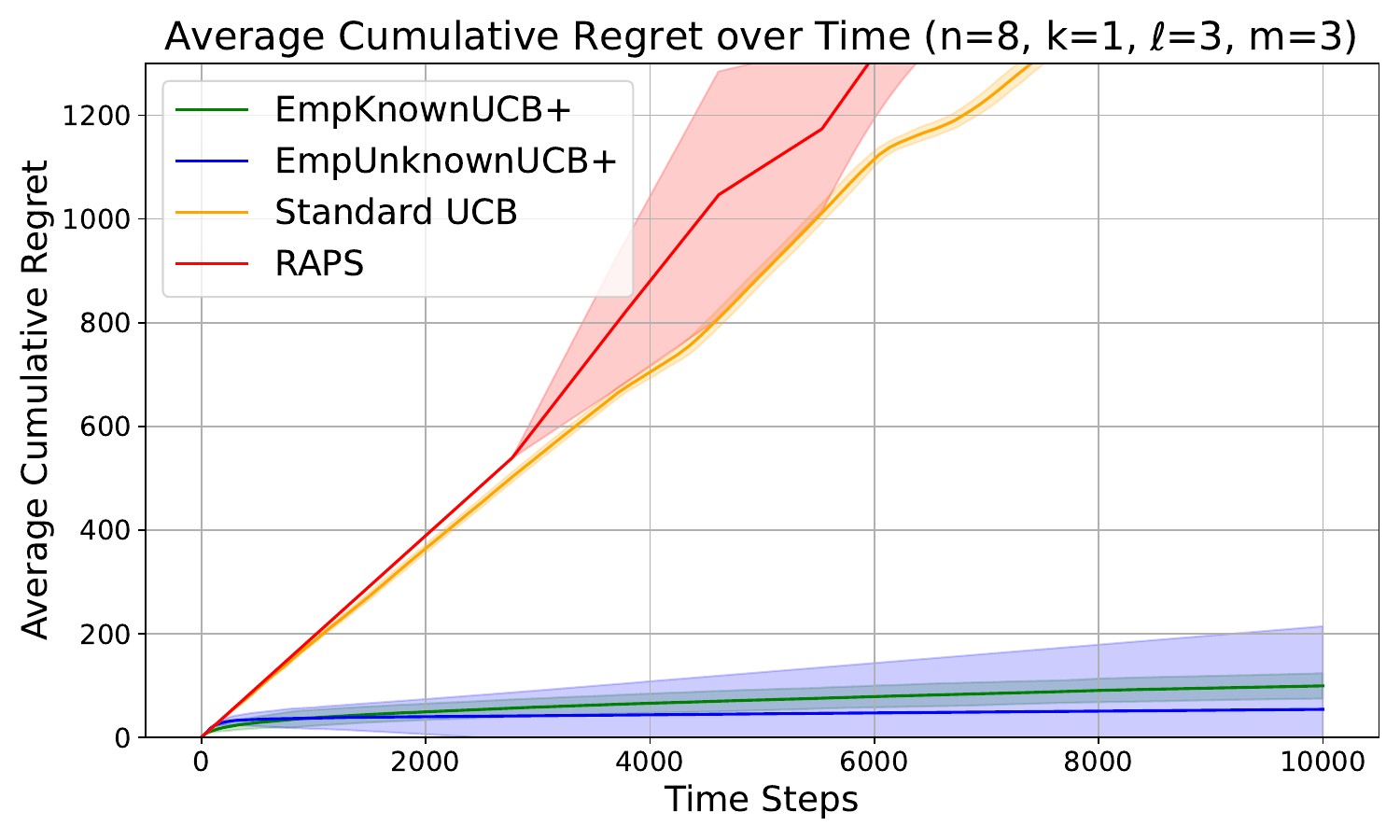}
        \caption{}
        \label{fig: varyT m3}
    \end{subfigure}\hfill
    \begin{subfigure}[t]{0.5\textwidth}
        \centering
        \includegraphics[width=\linewidth]{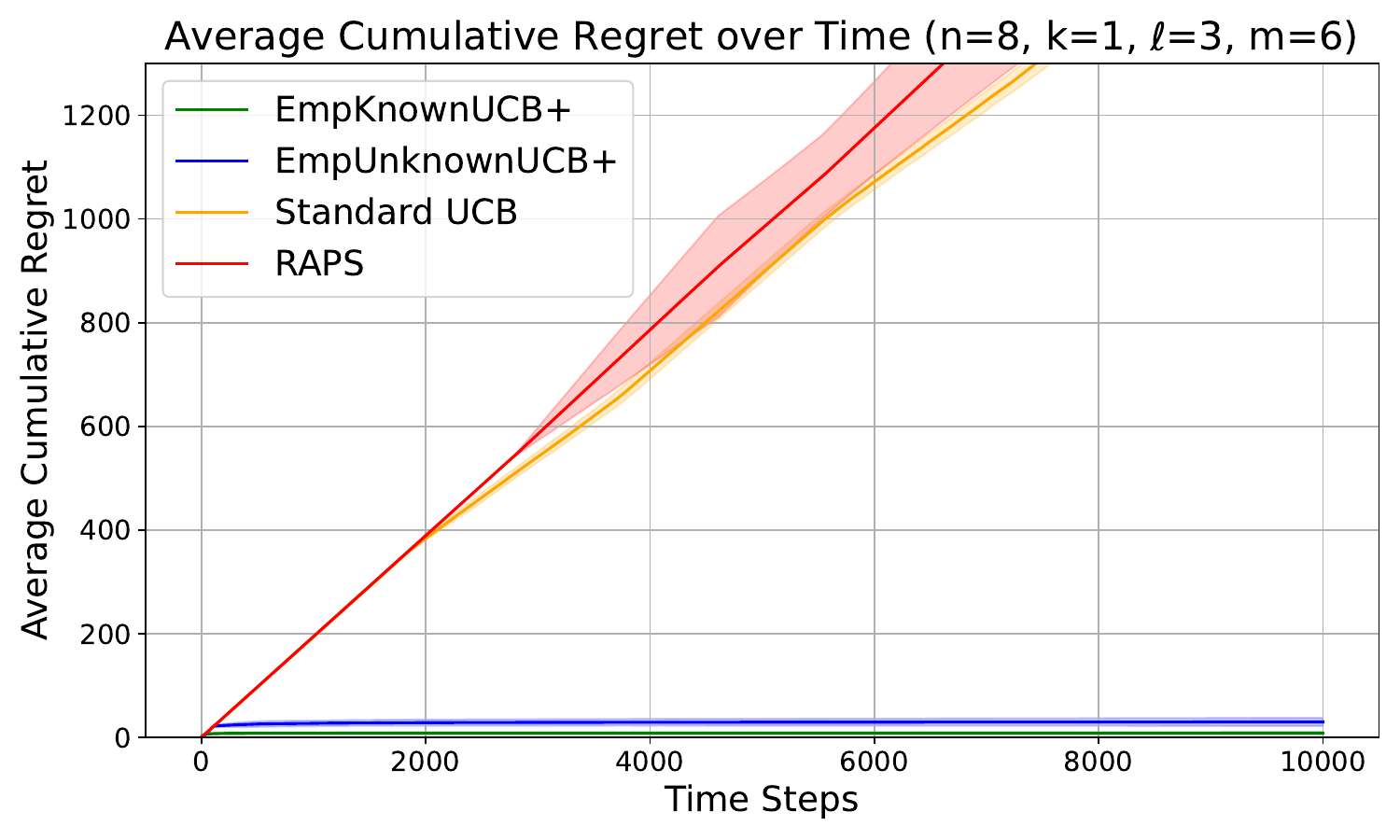}
        \caption{}
        \label{fig: varyT m6}
    \end{subfigure}\hfill
    \caption{Average regret of the algorithms over time. 
    (a) Corresponds to the setting with action size $m=3$. (b) Corresponds to the setting with action size $m=6$.}
    \label{fig: varyT}
\end{figure*}

Note that, similar to the known parent-size setting, the lower bound in Theorem~\ref{thm: unknown-k-lower-bound} applies to any policy in $\Pi(\{\G\})$. In contrast, Algorithm~\ref{algo: unknown-k} operates without knowledge of $\G$, yet, as shown in Lemma~\ref{lem: pareto optimality}, it remains close to Pareto optimal within $\Pi(\{\G\})$. This shows that even in the case of an unknown $k$, knowing the graph $\G$ does not significantly improve the rate of the worst-case regret.

\section{EXPERIMENTS}

In this Section, we evaluate the proposed algorithms on a variety of instances and compare their performance with existing methods. For all experiments, we generate Erdős–Rényi random graphs with edge probability $p = \tfrac{2}{n}$. The reward is modeled as a binary random variable whose mean is chosen uniformly at random from $[0,1]$ for each possible combination of parent values. Each experiment is repeated $100$ times; solid lines in the plots represent averages over these runs, while shaded regions indicate one standard deviation above and below the mean. Additional details on the setup of the experiments, such as the algorithm implementations, and further results are provided in the Appendix~\ref{sec: additional experiments}. The code is available at \hyperlink{https://github.com/ban-epfl/Unknown-Graph-Causal-Bandits}{https://github.com/ban-epfl/Unknown-Graph-Causal-Bandits}.

We compare against the following four algorithms:

\textbf{EmpKnownUCB+.} This is an empirical variant of Algorithm~\ref{algo: known-k}. In the experiments, if the algorithm chooses action $a_t = (p_t, s_t)$ at round $t$ and observes $(\mathbf{x}^{(t)}, y_t)$, then for any other action $a = (p,s) \in \A_m$ such that $\mathbf{x}^{(t)}_p = s$, we also treat $y_t$ as a reward sample for $a$. While this modification may lead to high regret in certain instances, our experiments indicate that it improves the average performance on random instances.

\textbf{EmpUnknownUCB+.} This is an empirical variant of Algorithm~\ref{algo: unknown-k} with two modifications: (i) samples collected in each phase are reused to construct confidence bounds for the arms in subsequent phases, and (ii) for phases $i > 1$, the algorithm selects $q_i$ arms with the highest empirical means, instead of choosing them uniformly at random.

\textbf{RAPS.} This algorithm, proposed in \cite{mikhail-konobeev2025causal}, represents the most recent approach for causal bandits with an unknown graph. It proceeds in two phases. Phase one is purely focused on parent identification, without regard for regret, and involves a sequential search procedure, only on atomic interventions, that identifies parents one by one. Phase two runs a standard UCB algorithm on the identified parent nodes. We additionally reuse the samples collected during phase one in phase two.

\textbf{Standard UCB.} This baseline algorithm simply runs the standard UCB algorithm on the entire set of actions.

In the following two experiments, we set $k=1$ and compare the regret of the algorithms over time and across different numbers of nodes. 
We fix $k=1$ to provide a favorable setting for the RAPS algorithm, since its first phase, which is dedicated entirely to parent identification, typically requires a large number of rounds and is repeated once per parent, thereby significantly increasing regret in practice. 
Although this is the most advantageous case for RAPS, our results demonstrate a substantial performance gap between our algorithms and both RAPS and Standard UCB. More experimental results are provided in Appendix \ref{sec: additional experiments}.

Figures~\ref{fig: varyT m3} and \ref{fig: varyT m6} report the average regret of the algorithms on an instance with parameters $n=8$, $k=1$, $\ell=3$, and intervention sizes $m=3,6$, for a horizon of $T = 10000$. 
Our two proposed algorithms perform nearly identically and achieve more than a $20\times$ improvement compared to the baselines. 
This close performance aligns with our theoretical analysis: for $k=1$, the regret bound of Algorithm~\ref{algo: unknown-k} coincides with that of Algorithm~\ref{algo: known-k}, and the additional adaptation cost appears only when $k > 1$.  

\begin{figure}
    \centering
    \includegraphics[width=\linewidth]{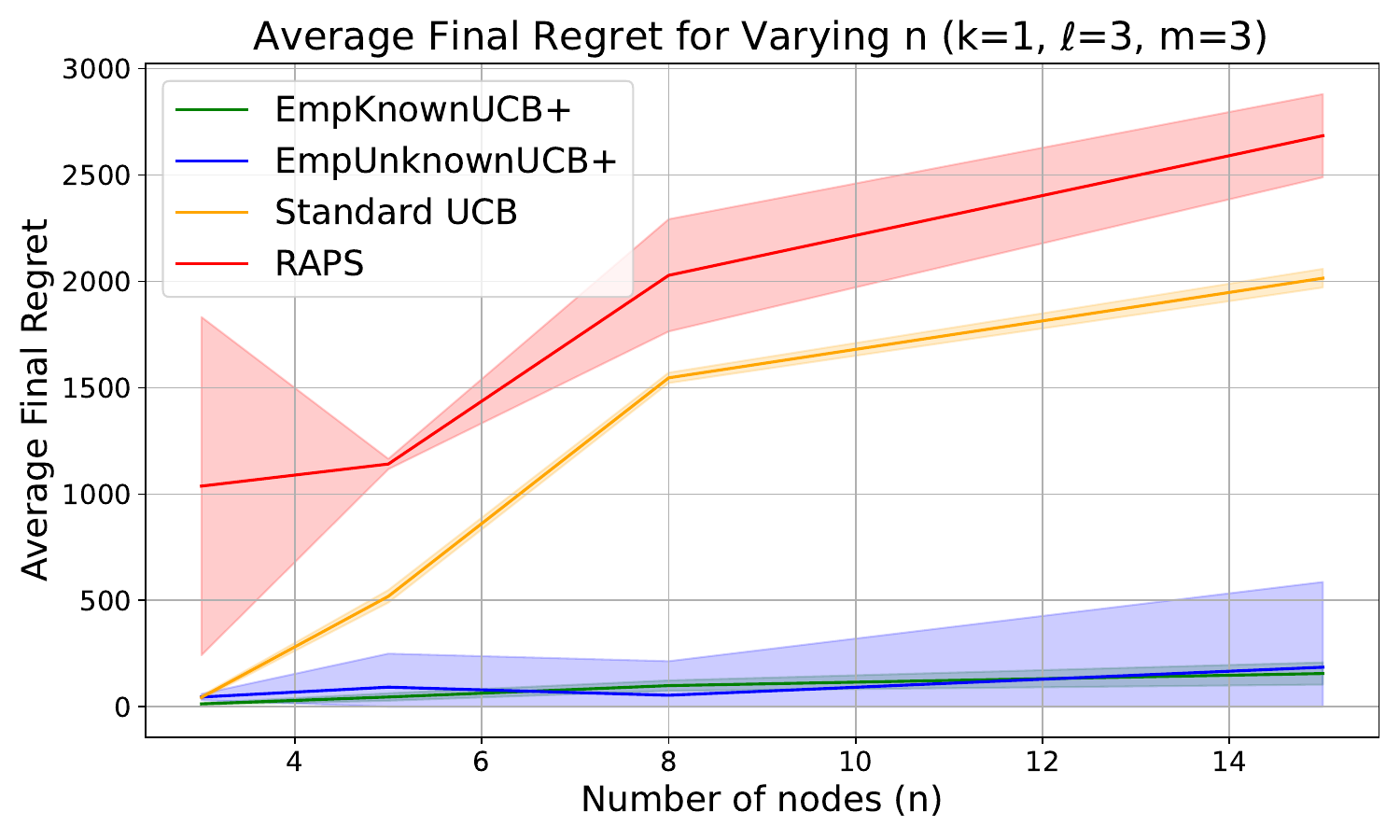}
    \caption{Average cumulative regret of algorithms at time $T$ for instances with varying numbers of nodes.}
    \label{fig: varyn}
\end{figure}

Figure~\ref{fig: varyn} shows the average final regret of the algorithms on instances with parameters $k=1$, $\ell=3$, for varying values of $n \in \{3,5,8,15\}$, with action size $m=3$ and horizon $T=10000$. The results again show that EmpKnownUCB+ and EmpUnknownUCB+ achieve very similar performance, with EmpUnknownUCB+ showing higher variance, while both methods substantially outperform the two baseline algorithms.

\section{DISCUSSION AND FUTURE WORK}

We studied causal bandits with unknown causal structure and showed that, under no distributional assumptions, worst-case regret minimization does not require identifying the reward’s parents. Our results establish the existence of a trade-off between structure learning and regret minimization, while our algorithms achieve nearly optimal rates without recovering the parent set. In both the known and unknown parent-size regimes, our algorithms provably outperform existing baselines and demonstrate strong empirical performance, highlighting that regret minimization can be addressed directly without explicit causal discovery.

While we prove that regret minimization and parent identification can be conflicting objectives in certain instances, characterizing this trade-off remains an interesting direction for future work. For example, one could study the settings in which these objectives are aligned and can be optimized simultaneously, or characterize achievable pairs of parent identification error rates and regret rates, and design algorithms that achieve such trade-offs.

We established that graph learning is not optimal for regret minimization without distributional assumptions. Still it would be valuable to investigate whether there exist realistic assumptions on the causal model that improve the regret. Since structural information alone does not significantly improve regret rates, future work should focus on distributional assumptions that could meaningfully enhance learning performance.



\bibliographystyle{alpha}
\bibliography{biblio}

\newpage

\onecolumn

\aistatstitle{Supplementary Materials}

\section{OMITTED PROOFS}

\subsection{Additional Notation and Helper Lemmas}

In this subsection, we introduce notation and a few helper quantities used in the proofs of the main theorems.

For any pair of probability measures $P,Q$, we denote their Kullback-Leibler divergence by $\kl{P,Q}$.
For any integer $r\ge 1$, define
\[
\mathcal{P}_{r}\; \coloneq  \;\bigl\{(p,s)\ \big|\ p\in \tbinom{[n]}{r},\ s\in [\ell]^r \bigr\},
\]
where $\tbinom{[n]}{r}$ denotes the set of all subset of $[n]$ of size $r$.
We use $\mathcal{P}_r$ to encode pairs $(p,s)$ where $p$ specifies the indices of $r$ variables (out of $n$) and
$s$ lists, in increasing index order, the special values assigned to those variables.  
In particular, the set of size $m$ actions $\mathcal{A}_m$ can be identified with $\mathcal{P}_m$.

For $r\ge 1$, any $(p,s)\in\mathcal{P}_r$, any policy $\pi$, and any horizon $t \in T$, let
$$
N_{p,s}(t,\pi)\;\coloneq \;\sum_{u=1}^{t}\,\mathbbm{1}_{\left\{\mathbf{x}^{(u)}_{p}=s\right\}},
$$
be the number of rounds up to time $t$ in which the coordinates indexed by $p$ take the value pattern $s$, and define the corresponding empirical mean reward
$$
\muh_{p,s}(t,\pi)\;\coloneq \;
\begin{cases}
\displaystyle \frac{1}{N_{p,s}(t,\pi)}\,
\sum_{u=1}^{t} Y_{u}\,\mathbbm{1}_{\left\{\mathbf{x}^{(u)}_{p}=s\right\}} \qquad & \text{if } N_{p,s}(t,\pi)>0,\\[2ex]
0 \qquad & \text{if } N_{p,s}(t,\pi)=0.
\end{cases}
$$
Here $(\mathbf{x}^{(u)},Y_u)$ denotes the observation at round $u$ along the trajectory induced by policy $\pi$ (we omit the instance for brevity). Intervened variables are counted with their assigned values.  
We similarly write $N_{a}(t,\pi)$ and $\muh_{a}(t,\pi)$ for any action $a\in\mathcal{A}$.  

When clear from context, we drop the explicit dependence on $\pi$ in each of these notations.


For any instance with $|\pa{Y}|=k$, let $\alpha_k$ denotes the fraction of optimal arms in $\A_m$. By Lemma~\ref{lem: optimal action A_m}, the optimal action lies in $\A_m$. For $m > k$, in $\A_m$ there are at least $\ell^{m-k}\binom{n-k}{m-k}$ arms achieving the optimal mean reward, since the reward depends only on the values of the $k$ parents and there are exactly this many interventions that set the values of all parents to their optimal combination. Thus, 
\begin{align} \label{eq: alpha_k value}
    \alpha_k 
    &= \frac{\ell^{m-k}\binom{n-k}{m-k}}{\ell^m \binom{n}{m}} 
    = \frac{\binom{m}{k}}{\ell^k \binom{n}{k}}.
\end{align}

For values of $k$ with $m < k$, we know one of the actions in $\A_m$ is optimal, thus $\alpha_k = \frac{1}{|\A_m|} = \frac{1}{\ell ^m \binom{n}{m}}$.

Now we provide two definitions from the literature on causal inference, which we will use in the proof of Lemma~\ref{lem: optimal action A_m}.

\begin{definition}[Blocked Path]
    Let $\G$ be a directed acyclic graph (DAG). A path between two nodes $X$ and $Y$ is said to be \emph{blocked} given a set of nodes $Z$ if at least one of the following holds:
    \begin{itemize}
        \item The path contains a chain $X_i \to X_j \to X_k$ or a fork $X_i \leftarrow X_j \to X_k$ such that the middle node $X_j \in Z$.
        \item The path contains a collider $X_i \to X_j \leftarrow X_k$ such that neither $X_j$ nor any of its descendants belong to $Z$.
    \end{itemize}
\end{definition}

\begin{definition}[d-separation] \label{def: d-separation}
    In a DAG $\G$, two disjoint sets of variables $A$ and $B$ are said to be \emph{d-separated} given a set of variables $Z$ if every path between a node in $A$ and a node in $B$ is blocked given $Z$. We denote this d-separation by $(A \perp \!\!\! \perp B | Z)_{\G}$. 
\end{definition}

\begin{lemma}[Pinsker Inequality] \label{lem: pinsker}
    For measures $P$ and $Q$ on the same probability space $(\Omega,\mathcal F)$
    \begin{equation*}
    \delta(P,Q)\ \coloneq \ \sup_{A\in\mathcal F}\, \big(P(A)-Q(A)\big)
    \ \le\ \sqrt{\frac{1}{2}\,\kl{P,Q}}.
    \end{equation*}
\end{lemma}

\begin{lemma} [\cite{bandit-book1-lattimore2020bandit}, Section $14$] \label{lem: measures distance}
    Let $(\Omega,\mathcal F)$ be a measurable space and let $P,Q:\mathcal F\to[0,1]$ be probability measures.
    Let $a<b$ and $X:\Omega\to[a,b]$ be an $\mathcal F$-measurable random variable, we have
    \begin{equation*}
    \left|\int_{\Omega} X(\omega)\,dP(\omega)-\int_{\Omega} X(\omega)\,dQ(\omega)\right|
    \ \le\ (b-a)\,\delta(P,Q),
    \end{equation*}
    where $\delta(P,Q)$ is as defined in Lemma~\ref{lem: pinsker}.
\end{lemma}

\begin{theorem}[Bretagnolle–Huber Inequality]\label{lem:Bretagnolle}
    Let $P$ and $Q$ be probability measures on the same measurable space $(\Omega, \mathcal{F})$, and let $A\in\mathcal{F}$ be an arbitrary event. Then,
    \begin{align*}
        P(A)+Q(A^c)\geq \frac{1}{2}\exp \big( -\kl{P,Q} \big),
    \end{align*}
    where $A^c=\Omega\setminus A$.
\end{theorem}

\begin{lemma}\label{lem: sub-gaussian}
    Let $X_1, X_2, \ldots, X_k$ be sub-Gaussian random variables with parameter $\sigma^2 = 1$ and means $\mathbb{E}[X_i] = \mu_i \in [a, b]$. 
    Define the mixture random variable $X$ such that $X = X_i$ with probability $p_i$. 
    Then $X$ is sub-Gaussian with parameter 
    \[
        \sigma_X^2 = 1 + \frac{(b - a)^2}{4}.
    \]
\end{lemma}
\begin{proof}
    We can decompose $X$ as $X = Y + \epsilon$, where $Y$ is a discrete random variable taking value $\mu_i$ with probability $p_i$, and $\epsilon$ is a $1$-sub-Gaussian random variable independent of $Y$. 
    The random variable $Y$ is $\tfrac{(b - a)}{2}$-sub-Gaussian because it is supported on an interval of length $(b - a)$. 
    Since $\epsilon$ is $1$-sub-Gaussian, and the sum of independent $\sigma_1$- and $\sigma_2$-sub-Gaussian random variables is $\sqrt{\sigma_1^2 + \sigma_2^2}$-sub-Gaussian, it follows that
    \[
        \sigma_X^2 = 1 + \frac{(b - a)^2}{4}.
    \]
\end{proof}

\begin{lemma} \label{lem: non-optimal prob}
    Let $\alpha$ denote the fraction of optimal arms in $\A_m$, i.e., the number of optimal arms divided by $|\A_m|$. Then, if we choose a random subset of arms of size $\frac{1}{\alpha} \ln \sqrt{T}$, either with or without replacement, the probability that the subset contains no optimal arm is less than $\tfrac{1}{\sqrt{T}}$.
\end{lemma}
\begin{proof}
    Let $N = |\A_m|$ be the total number of arms, so the number of optimal arms is $\alpha N$. Consider sampling with replacement. In each draw, the probability of picking a non-optimal arm is $1-\alpha$. After 
    $$
    q \;\coloneq \; \frac{1}{\alpha}\ln\sqrt{T}
    $$ 
    independent draws, the probability that all sampled arms are non-optimal is
    $$
    (1-\alpha)^q \;\leq\; \exp(-\alpha q) \;=\; \exp\!\big(-\ln\sqrt{T}\big) \;=\; \frac{1}{\sqrt{T}},
    $$
    where the inequality comes from the fact that $\forall x \in \mathbb{R} : 1 - x \leq \exp (-x)$.
    Thus, with replacement, the probability that the sampled set contains no optimal arm is at most $1/\sqrt{T}$. 
    
    Sampling without replacement can only decrease the probability of missing all optimal arms (it is stochastically dominated by the with-replacement model), so the same bound applies.
\end{proof}

\subsection{Proof of Lemma~\ref{lem: optimal action A_m}}
In this section, we present the proof of Lemma \ref{lem: optimal action A_m}. 

\optimalActionAm*

\begin{proof}
    To prove this lemma, we separate two cases depending on whether $m \geq k$ or $m < k$. 

    \textbf{Case 1: $m \geq k$.} In this case, for any action $a = (p,s) \in \A$, we have 
    \begin{align*}
        \mu_a = \mathbb{E} \big[ Y \mid do(\mathbf{X}_p = s)\big] 
        &= \sum_{\mathbf{z} \in [\ell]^k} \pr \big( \pa{Y} = \mathbf{z} \mid do(\mathbf{X}_p = s) \big) \, \mathbb{E} \big[Y \mid \pa{Y} = \mathbf{z} \big]  \\
        &\leq \max_{\mathbf{z} \in [\ell]^k} \mathbb{E} \big[Y \mid \pa{Y} = \mathbf{z} \big] = \max_{\mathbf{z} \in [\ell]^k} \mathbb{E} \big[Y \mid do(\pa{Y} = \mathbf{z}) \big]\\
        &= \mu_{a_{\mathbf{z}}},
    \end{align*}
    where $a_{\mathbf{z}}$ is any of the actions in $\A_m$ that intervenes on all the reward parents and sets their values to $\mathbf{z}$.

    \textbf{Case 2: $m < k$.} This case requires a more involved argument. We show that for any action $a = (p,s)$, there exists an intervention that intervenes on one more variable than $a$ and achieves mean reward at least $\mu_a$. Repeating this process yields an action of size $m$ with mean reward at least as large as $\mu_a$.

    Fix any $r \notin p$. Then
    \begin{align*}
        \mu_a &= \mathbb{E} \big[ Y \mid do(\mathbf{X}_p = s)\big] \\
        &= \sum_{i \in [\ell]} \pr \big( X_r = i \mid do(\mathbf{X}_p = s) \big) \, \mathbb{E} \big[Y \mid do(\mathbf{X}_p = s), X_r = i \big].
    \end{align*}
    If we can show that 
    $$
    \pr \big(Y \mid do(\mathbf{X}_p = s), X_r = i \big) = \pr \big(Y \mid do(X_p = s, X_r = i) \big) \quad \forall i \in [\ell],
    $$
    then it follows that
    \begin{align*}
        \mu_a &= \sum_{i \in [\ell]} \pr \big( X_r = i \mid do(\mathbf{X}_p = s) \big) \, \mathbb{E} \big[Y \mid do(X_p = s, X_r = i) \big] \\
        &\leq \max_{i \in [\ell]} \mathbb{E} \big[Y \mid do(X_p = s, X_r = i) \big] = \mu_{a'},
    \end{align*}
    where $a'$ is an action that extends $a$ by also intervening on $X_r$. Thus, it suffices to show that such an $X_r$ exists.

    By the second rule of Pearl’s do-calculus~\cite{pearl2009causality}, a sufficient condition for the above equality is that 
    \begin{align} \label{eq: d-sep X_r}
        (Y \perp\!\!\!\perp X_r \mid X_p)_{\G'_{\overline{X_p}\underline{X_r}}},
    \end{align}
    where $\G' = \G \cup \{Y\}$, and ${\G'_{\overline{X_p}\underline{X_r}}}$ denotes the graph $\G'$ with all incoming edges to $X_p$ removed and all outgoing edges from $X_r$ removed. 

    To prove~\eqref{eq: d-sep X_r}, let $\sigma : [n] \to [n]$ be a topological ordering of the DAG $\G$, i.e., in the sequence $X_{\sigma(1)}, X_{\sigma(2)}, \ldots, X_{\sigma(n)}$ all edges point forward. Choose $r$ as the first index in this order such that $r \notin p$. We claim this choice of $X_r$ satisfies~\eqref{eq: d-sep X_r}. 

    Consider any path $X_r - X_{i_1} - X_{i_2} - \ldots - X_{i_t} - Y$ in ${\G'_{\overline{X_p}\underline{X_r}}}$. This path has the following properties:  
    \begin{enumerate}
        \item There must be at least one intermediate variable other than $X_r$ and $Y$ (i.e., $t>0$), because all outgoing edges from $X_r$ are removed and $Y$ has no children. Thus $X_r \to Y$ and $X_r \leftarrow Y$ are impossible.  
        \item The first edge must be oriented $X_r \leftarrow X_{i_1}$, since no outgoing edges from $X_r$ remain in ${\G'_{\overline{X_p}\underline{X_r}}}$.  
        \item By the causal order and the choice of $r$, we must have $X_{i_1} \in X_p$ or there is no path between $X_r$ and $Y$.  
    \end{enumerate}
    Therefore, $X_{i_1}$ lies on the path, is conditioned on, and blocks the path regardless of the orientation of the next edge. Hence all such paths are blocked, establishing~\eqref{eq: d-sep X_r}. 

    This proves that for any action $a$ there exists an extended action $a'$ with $\mu_{a'} \geq \mu_a$. Repeating this argument iteratively shows that there exists an optimal action in $\A_m$, completing the proof.
\end{proof}

\subsection{Proofs of Section \ref{sec: trade-off}}\label{proof:id-regret-trade-off}

Herein, we present the proofs for Section~\ref{sec: trade-off} of the main text. 


\tradeOff* 

\begin{proof}
    To prove this result, we first introduce the class of instances $\E_0$. The definition of $\E_0$ proceeds in three steps: we specify the common graph $\G$ over $\X$, then the conditional distributions of variables given their parents, and finally the parent set $\pay$ together with the reward distribution. Throughout, we assume all variables are binary.

    The graph $\G$ shared among all instances in $\E_0$ is defined as follows. For every $i \in \{1,2,\ldots,k\}$ and every $j \in \{k+1,\ldots,n\}$, $X_i$ is a parent of $X_j$, and there are no other edges between pairs of nodes. This structure is illustrated in Figure~\ref{fig: trade-off-graph}.  
    
    For the conditional distributions, the variables $X_i$ with $i \in [k]$ have no parents and are independent. We fix $\pr(X_i = 0) = 1$ for each such variable, meaning they always take the value $0$ unless intervened upon. For any $i > k$, we define
    $$
    \pr \paran{ X_i = 1 \mid X_1 = x_1, X_2 = x_2, \ldots, X_k = x_k } = 
    \begin{cases}
        0 & \text{if } x_1 x_2 \cdots x_k = 0, \\
        1 & \text{if } x_1 x_2 \cdots x_k = 1.
    \end{cases}
    $$
    That is, these variables are always equal to $0$, except in the case where all nodes $\{X_1, X_2, \ldots, X_k\}$ equal $1$, in which case they take the value $1$.
    
    The reward variable $Y$ has distribution $\mathcal{N}(0,1)$ for any combination of parents except the case where all parents equal $1$, in which case its distribution is $\mathcal{N}(1,1)$. Thus, the mean reward is $0$ in all cases except when every parent is set to $1$, where the mean is $1$.
    
    All instances in $\E_0$ share the same graph $\G$, the distributions of all variables, and the reward distribution. The only difference between them is the identity of the parent set $\pay$. For each set $p \in \binom{[n]}{k}$, we construct an instance $\V_p \in \E_0$ where $\pay = p$. Hence, the total number of instances is $|\E_0| = \binom{n}{k}$.

    \begin{figure}[t] 
    \centering
    \includegraphics[width=0.8\textwidth]{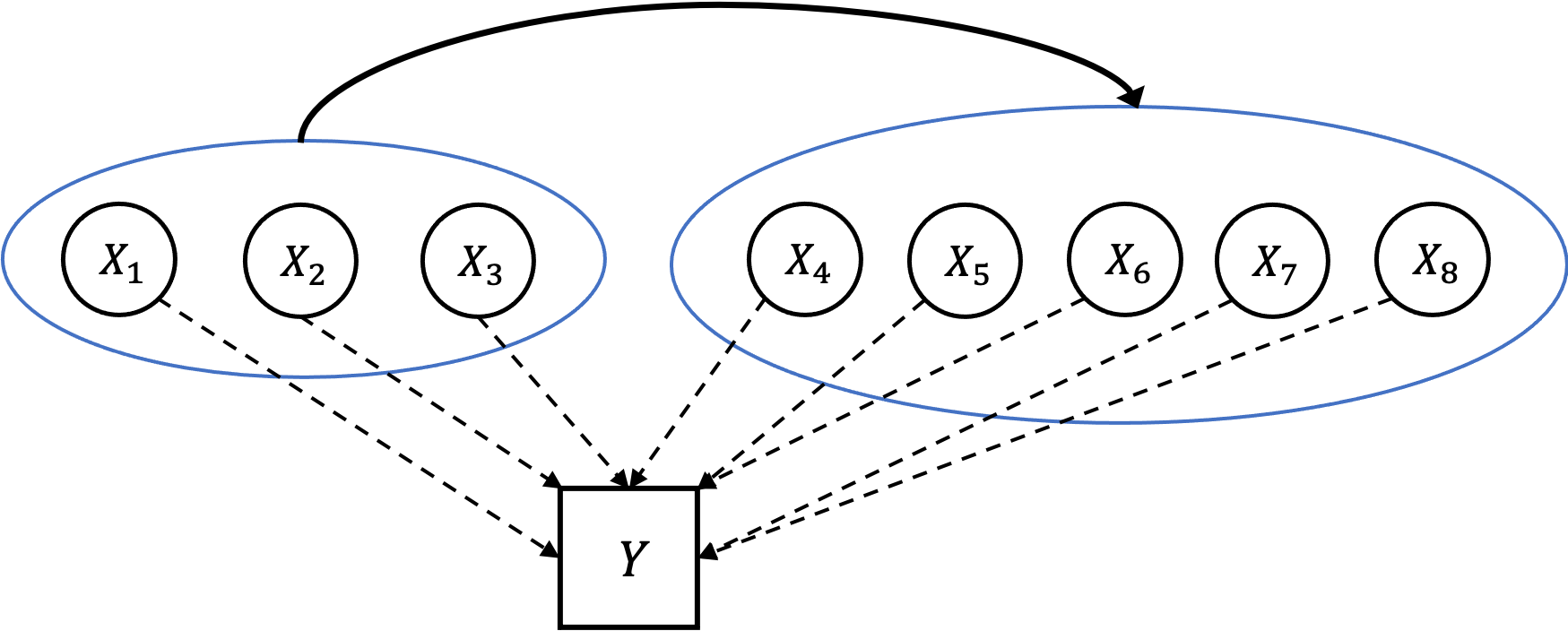}
    \caption{Graph structure for $n=8$, $k=3$. Nodes $X_1, X_2, X_3$ act as parents of all other variables. The reward node $Y$ has dashed incoming edges from all variables, indicating that any subset of size $k$ can form the true parent set $\pay$.}\label{fig: trade-off-graph}
    \end{figure}

    We now prove that in the setting $m = k$, uniform sampling, combined with a suitable suggestion rule $\bar{d}_{\text{Unif}}$, achieves parent identification error $\ocal\paran{\exp\paran{-T}}$ on instances in $\E_0$.

    \paragraph{Decision Rule $\bar{d}_{\text{Unif}}$.}
    Given the empirical means $\muh_a(T)$ for $a \in \A_k$ (policy/instance dependence suppressed), the rule is:
    \begin{enumerate}
        \item If there exists an action $a = (p_a, s_a)$ such that
        \begin{itemize}
            \item $\muh_a(T) > 0.5$, and
            \item $p_a \neq (1,2,\ldots,k)$,
        \end{itemize}
        then output $\bar{d}_{\text{Unif}} = p_a$ (if there are multiple such actions, choose one of them arbitrarily).
        \item Otherwise, output $\bar{d}_{\text{Unif}} = [k]$.
    \end{enumerate}
    
    Consider the uniform sampling policy that samples each arm in $\A_m$ uniformly, with $m = k$, i.e., each $a \in \A_k$ is pulled exactly $T / |\A_k|$ times (up to $\pm 1$). We show that applying $\bar{d}_{\text{Unif}}$ after $T$ rounds yields an error probability $\ocal\paran{\exp\paran{-T}}$ on $\E_0$.
    
    \begin{proposition}\label{prop:uniform-exp}
    For the class $\E_0$ described above and $m=k$, the uniform-sampling policy on $\A_k$ combined with the rule $\bar{d}_{\text{Unif}}$ satisfies
    $$
    \delta \paran{\pi_{\text{unif}},\bar{d}_{\text{Unif}}, \E_0} \le C \exp\paran{- c T},
    $$
    for some constants $C,c>0$ that may depend on $n,k$ but not on $T$.
    \end{proposition}

    \begin{proof}
    Fix an instance $\V_p \in \E_0$ with true parent set $\pay=p$. Recall from the definition of $\E_0$ (binary variables) that the reward has mean $1$ if and only if \emph{all} parents in $\pay$ equal $1$, and mean $0$ otherwise. Let
    $$
    a_p^\star \;=\; \big(p,\mathbf{1}_k\big), \qquad 
    a_0 \;=\; \big([k],\mathbf{1}_k\big),
    $$
    where $\mathbf{1}_k$ is the all-ones vector in $\{0,1\}^k$. Under $\E_0$, the only actions with $\mu_a=1$ are $a_p^\star$ and $a_0$; every other $a\in\A_k$ has $\mu_a=0$. Indeed, setting $\mathbf{X}_{[k]}$ to $1$ forces all $X_j$, $j>k$, to $1$, so $a_0$ ensures that \emph{any} parent set $\pay$ is all ones; conversely, if an action does not set all coordinates in $p$ to $1$ and does not set $\mathbf{X}_{[k]}$ to $1$, then at least one parent remains $0$ (since $X_i=0$ for each $i$ unless intervened or the action is $a_0$), hence $\mu_a=0$.
    
    Under uniform sampling,
    $$
    N_a(T) \in \Big\{\Big\lfloor \tfrac{T}{|\A_k|}\Big\rfloor, \Big\lceil \tfrac{T}{|\A_k|}\Big\rceil\Big\} 
    \quad\text{so in particular}\quad
    N_{\min} \;\coloneq \; \min_{a\in\A_k} N_a(T) \;\ge\; \frac{T-|\A_k|}{|\A_k|}.
    $$
    Since rewards are $1$-sub-Gaussian with means $\mu_a\in\{0,1\}$, by Hoeffding's inequality, for any $\epsilon\in(0,1/2)$,
    $$
    \pr\!\paran{\big|{\muh_a(T)-\mu_a} \big|>\epsilon} \;\le\; 2 \exp\!\paran{- c_0 N_a \epsilon^2}
    \quad\text{for all } a\in\A_k
    $$
    for some absolute constant $c_0>0$. By a union bound over $\A_k$ and the bound on $N_{\min}$,
    $$
    \pr\!\paran{\exists a\in\A_k:\big| {\muh_a(T)-\mu_a} \big| >\epsilon}
    \;\le\; 2|\A_k|\exp\!\paran{- c_0 N_{\min}\epsilon^2}
    \;\le\; 2|\A_k|\exp\!\paran{- c_0 \epsilon^2 \,\frac{T-|\A_k|}{|\A_k|}}.
    $$
    Fix $\epsilon=\tfrac{1}{4}$. For all $T\ge 2|\A_k|$ we have $(T-|\A_k|)/|\A_k|\ge T/(2|\A_k|)$, hence
    $$
    \Pr\!\paran{\exists a\in\A_k:\big|{\muh_a(T)-\mu_a} \big| >\tfrac{1}{4}}
    \;\le\; C_1 \exp\!\paran{- c_1 T},
    $$
    with $C_1=2|\A_k|$ and $c_1 = c_0/(32|\A_k|)$. (For $T<2|\A_k|$ the bound can be absorbed into constants.)
    
    Now we show that under the complement of this event, i.e., $\forall a \in \A_k : \big|{\muh_a(T)-\mu_a} \big| \leq \tfrac{1}{4}$, the parent set is identified correctly by $\bar{d}_{\text{Unif}}$. Under this good event, we have
    $$
    \muh_{a_p^\star}(T) \ge 1-\tfrac{1}{4} > \tfrac{1}{2}, 
    \qquad
    \muh_{a_0}(T) \ge 1-\tfrac{1}{4} > \tfrac{1}{2},
    \qquad
    \muh_a(T) \le \tfrac{1}{4} < \tfrac{1}{2} \ \ \forall a\notin\{a_p^\star,a_0\}.
    $$
    \emph{Case 1: $p\neq [k]$.} Since $p_{a_p^\star}=p\neq [k]$, $\muh_{a_p^\star}(T) > \tfrac{1}{2}$, and $\muh_a(T) < \tfrac{1}{2} \ \ \forall a\notin\{a_p^\star,a_0\}$, the decision rule outputs $\bar{d}_{\text{Unif}}=p$.
    
    \emph{Case 2: $p=[k]$.} Then $a_p^\star=a_0$, and under the good event there is no action with $\muh_a(T)>\tfrac{1}{2}$ and $p_a\neq [k]$. The rule therefore outputs $\bar{d}_{\text{Unif}}=[k]=p$.
    
    In both cases, under the good event the output is correct; thus
    $$
    \pr\!\paran{\bar{d}_{\text{Unif}}\neq p} \;\le\; \pr\!\paran{\exists a\in\A_k:\big| {\muh_a(T)-\mu_a} \big| >\tfrac{1}{4}}
    \;\le\; C_1 \exp\!\paran{- c_1 T}.
    $$
    Renaming constants gives the claimed bound with some $C,c>0$ depending at most on $n,k$ (through $|\A_k|=\binom{n}{k} 2^k$), but not on $T$.
    \end{proof} 

    We now prove the second statement of the theorem. Fix any policy $\pi$. For any set $p \in \binom{[n]}{k}$, recall that $\pr_p^\pi$ denotes the probability measure induced by $T$ rounds of interaction between $\pi$ and $\V_p$. Define the event $\E$ as the event that the estimated parent set of $\pi$, denoted by $\hat{Pa}(\pi)$, is equal to $[k]$. Then, by the definition of $\delta(\pi, \bar{d}, \E_0)$, for some decision rule notation $\bar{d}$, we have 
    $$
    \forall p \neq [k] : \quad 2 \delta(\pi,\bar{d}, \E_0) \;\geq\; \pr^\pi_{[k]}\paran{\E^c} + \pr^\pi_p(\E).
    $$
    We omit the decision rule for simpler representation. 
    By Bretagnolle--Huber’s inequality \ref{lem:Bretagnolle}, it follows that
    \begin{align*}
    \pr^\pi_{[k]}\paran{\E^c} + \pr^\pi_p(\E) \;\geq\; \tfrac{1}{2} \exp \paran{- \kl{\pr^\pi_{[k]}, \pr^\pi_p}},
    \end{align*}
    and therefore
    \begin{align} \label{eq: trade-off-delta-kl}
    4 \delta(\pi, \E_0) \;\geq\; \exp \paran{- \kl{\pr^\pi_{[k]}, \pr^\pi_p}}.
    \end{align}
    
    To compute the KL divergence between $\pr^\pi_{[k]}$ and $\pr^\pi_p$, observe that the only difference between these two instances arises when the action is $a_p = (p, \mathbf{1}_k)$, i.e., the intervention that sets all variables in $p$ to $1$. In this case, the reward distribution is $\mathcal{N}(1,1)$ in $\V_p$ but $\mathcal{N}(0,1)$ in $\V_{[k]}$. For any other action $a$, the joint distribution of $(X_1,\ldots,X_n,Y)$ is identical under both instances, so the contribution to the KL divergence vanishes. 

    In this setting, when the algorithm is playing on $\V_{[k]}$, the only informative action for the parent identification task is $a_p$, since any other action provides no information to distinguish between $\V_{[k]}$ and $\V_p$. However, in $\V_{[k]}$ this action incurs a regret of $1$. On the other hand, the only optimal action in $\V_{[k]}$ is $a_0 = ([k], \mathbf{1}_k)$, which provides no information about the true parent set. Since the set of optimal actions and the set of informative actions for parent identification are disjoint, this illustrates the fundamental trade-off between regret minimization and parent identification. 
    
    By the divergence decomposition lemma (Lemma~15.1 in \cite{bandit-book1-lattimore2020bandit}), we obtain
    $$
    \kl{\pr^\pi_{[k]}, \pr^\pi_{p}} \;=\; \mathbb{E}_{[k]}\!\big[N_{a_p}(T)] \; \kl{\mathcal{N}(0,1), \mathcal{N}(1,1)},
    $$
    where $\mathbb{E}_{[k]}$ denotes expectation with respect to $\pr^\pi_{[k]}$. Since 
    $$
    \kl{\mathcal{N}(1,1), \mathcal{N}(0,1)}=\tfrac{1}{2},
    $$ 
    we conclude that
    $$
    \forall p \neq [k] : \quad \kl{\pr^\pi_{[k]}, \pr^\pi_{p}} \;=\; \tfrac{1}{2}\,\mathbb{E}_{[k]}\!\big[N_{a_p}(T)\big].
    $$
    
    Now suppose $\delta(\pi, \E_0) \in \ocal \paran{\exp(-T^{\alpha})}$, meaning that there exist constants $C,c>0$ such that 
    $$
    \delta(\pi, \E_0) \leq C \exp(-c T^{\alpha}).
    $$ 
    By~\eqref{eq: trade-off-delta-kl}, we then have
    \begin{align*}
    C \exp(-c T^{\alpha}) 
    &\;\geq\; \exp\!\left(-\tfrac{1}{2}\,\mathbb{E}_{[k]}\!\big[N_{a_p}(T)\big]\right) \\
    &\;\Longrightarrow\; \ln(C) - c T^{\alpha} \;\geq\; -\tfrac{1}{2}\,\mathbb{E}_{[k]}\!\big[N_{a_p}(T)\big] \\
    &\;\Longrightarrow\; \mathbb{E}_{[k]}\!\big[N_{a_p}(T)\big] \;\geq\; C' T^{\alpha},
    \end{align*}
    for some constant $C'>0$. 
         
    Now, the key point is that each time the algorithm plays $a_p$ in $\V_{[k]}$, it incurs a regret of $1$, since $a_p$ is strictly suboptimal there. Therefore, the total expected regret is proportional to the number of pulls of $a_p$. Combining this with the bound above, we obtain
    $$
    R_T(\pi, \E_0) \;\geq\; R_T(\pi, \V_{[k]}) 
    \;\geq\; \mathbb{E}_{[k]}\!\big[N_{a_p}(T)\big] \times 1 
    \;\geq\; C' T^{\alpha},
    $$
    
    which completes the proof.

\end{proof}

\subsection{Proofs of Section \ref{sec: known-k}}


\mgeqkLB* 

\begin{proof}
    First note that the empty graph is a subgraph of any graph $\G$. Moreover, for any graph $\G$ over variables $X_1,\ldots,X_n$, we can specify a data-generating distribution that \emph{ignores} any given edge (i.e., the conditional distribution of a child either does not change with its parent or changes arbitrary small), effectively removing that edge. Hence, to prove the desired statement for arbitrary $\G$, it suffices to show it for the empty graph on $\{X_1,\ldots,X_n\}$; the result then extends to any $\G$ by appropriately “turning off’’ edges. We therefore assume from now on that the graph is empty, so all variables are independent.

    Let $\V$ denote the neutral instance in which the graph $\G$ is empty, the reward distribution is $\N(0,1)$ for any values of its parents (indeed, the distribution is constant so the choice of the parent set can be any $k$-subset), and
    $$
    \forall i \in [n]: \quad \Pr(X_i=1)=1,
    $$
    so that each $X_i$ equals $1$ unless it is intervened upon. Fix any policy $\pi$ interacting with $\V$ for $T$ rounds, and let $\pr_{\V}$ denote the probability measure induced by this sequential interaction.
    
    For any pair $(p,s)$ with $p \in \binom{[n]}{k}$ and $s \in [\ell]^k$, define a perturbed instance $\vps$ which is identical to $\V$ in the graph and in the distributions of the non-reward variables. In $\vps$, the reward’s parent set is $p$, and for any assignment $\mathbf{b} \in [\ell]^k$ of the parents we set
    $$
    Y \mid \mathbf{X}_p=\mathbf{b} \sim \N\!\big(\mu(\mathbf{b}),\,1\big),
    \qquad
    \text{where}\quad
    \mu(\mathbf{b}) \;=\;
    \begin{cases}
    \Delta, & \mathbf{b}=s,\\[2pt]
    0, & \text{otherwise,}
    \end{cases}
    $$
    where $\Delta$ is a positive real number to be chosen later.

    Let $\pr_{p,s}$ denote the probability measure obtained by $T$ rounds of interaction between $\pi$ and $\vps$. By setting $P = \pr_{\V}, Q = \pr_{p,s}$ in Lemma~\ref{lem: measures distance}, and letting $\tps$ denote the random variable equal to the number of times that $\mathbf{x}_p = s$ during the $T$ rounds, i.e. $\tps = \sum_{t \in [T]} \mathbbm{1}_{\left\{ \mathbf{x}^{(t)}_p = s \right\}}$, we obtain
    \begin{align*}
        \big| \mathbb{E}_{\V}[\tps] - \mathbb{E}_{p,s}[\tps] \big| \;\leq\; T \, \delta(\pr_{\V}, \pr_{p,s}),
    \end{align*}
    where we use the fact that $\tps \in [0,T]$.  
    
    Then, by Lemma~\ref{lem: pinsker}, it follows that 
    \begin{align} \label{eq: expected T difference}
        \mathbb{E}_{p,s}[\tps] 
        \;\leq\; \mathbb{E}_{\V}[\tps] \;+\; T \, \sqrt{\tfrac{1}{2} \, \kl{\pr_{\V}, \pr_{p,s}}}.
    \end{align}
    
    Next, note that the observational distribution of all variables is identical in $\V$ and $\vps$, except when $\mathbf{X}_p = s$. Therefore, by the divergence decomposition lemma (Lemma~15.1 in \cite{bandit-book1-lattimore2020bandit}), we obtain
    $$
    \kl{\pr_{\V}, \pr_{p,s}} \;=\; \mathbb{E}_{\V}[\tps] \, \kl{\N(0,1), \N(\Delta,1)}.
    $$
    Since $\kl{\N(0,1), \N(\Delta,1)} = \tfrac{\Delta^2}{2}$, inequality~\eqref{eq: expected T difference} yields
    \begin{align*} 
        \mathbb{E}_{p,s}[\tps] 
        \;\leq\; \mathbb{E}_{\V}[\tps] \;+\; \frac{T}{2} \, \sqrt{\mathbb{E}_{\V}[\tps] \Delta^2}.
    \end{align*}

    Fix any arbitrary subset $\J \subseteq \pcal_k$ consisting of pairs $(p,s)$. Let 
    $$
    \wps \;=\; \frac{1}{T} \, \mathbb{E}_{\V}[\tps]
    $$ 
    denote the fraction of time that policy $\pi$ observes $\mathbf{X}_p = s$ during the process. Then, the above inequality implies
    \begin{align*}
        \sum_{(p,s) \in \J} \mathbb{E}_{p,s}[\tps]  
        &\leq T \sum_{(p,s) \in \J} \wps 
        \;+\; \frac{T \Delta}{2} \sum_{(p,s) \in \J} \sqrt{T \wps}  \\ 
        &\leq T \sum_{(p,s) \in \J} \wps 
        \;+\; \frac{T \Delta}{2} \sqrt{T \, |\J| \sum_{(p,s) \in \J} \wps},
    \end{align*}
    where the last inequality follows from the Cauchy–Schwarz inequality.  
    
    For any pair $(p,s)$, note that 
    $$
    R_T(\pi, \V_{p,s}) \;=\; \Delta \big(T - \mathbb{E}_{p,s}[\tps]\big).
    $$ 
    Therefore,
    \begin{align*}
        \sum_{(p,s) \in \J} R_T(\pi, \V_{p,s}) 
        &= \Delta T |\J| - \Delta \sum_{(p,s) \in \J} \mathbb{E}_{p,s}[\tps] \\ 
        &\geq \Delta T |\J| - \Delta T \sum_{(p,s) \in \J} \wps 
        - \frac{\Delta^2 T}{2} \sqrt{T \, |\J| \sum_{(p,s) \in \J} \wps}.
    \end{align*}
    
    By dividing both sides by $|\J|$, we obtain
    \begin{align*}
        \exists (p,s) \in \J : \quad 
        R_T(\pi, \V_{p,s}) 
        &\geq \Delta T \left(1 - \frac{\sum_{(p,s) \in \J} \wps}{|\J|} 
        - \frac{\Delta}{2} \sqrt{T \, \frac{\sum_{(p,s) \in \J} \wps}{|\J|}} \right).
    \end{align*}
    
    Since $\forall (p,s): \wps \leq 1$, we have $\tfrac{\sum_{(p,s) \in \J} \wps}{|\J|} \leq 1$. For any $\J$, let $c_{\J}$ denote the maximum possible value of $\tfrac{\sum_{(p,s) \in \J} \wps}{|\J|}$ over all policies. By setting
    $$
    \Delta \;=\; (1 - c_{\J}) \sqrt{\frac{|\J|}{T \sum_{(p,s) \in \J} \wps}},
    $$
    we obtain
    \begin{align*}
        \exists (p,s) \in \J : \quad 
        R_T(\pi, \V_{p,s}) 
        &\geq (1 - c_{\J}) \sqrt{\frac{T |\J|}{\sum_{(p,s) \in \J} \wps}} \cdot \frac{1 - c_{\J}}{2} \\
        &\in \Omega \!\left((1 - c_{\J})^2 \sqrt{\frac{T |\J|}{\sum_{(p,s) \in \J} \wps}}\right).
    \end{align*}
    
    Since this holds for any subset $\J \subseteq \pcal_k$, we arrive at the following general lower bound, which to our knowledge is novel in the literature:
    \begin{lemma}[General Lower Bound] \label{lem: general-lb}
    For any graph $\G$ and any policy $\pi \in \Pi(\{k,\G\})$, we have
    \begin{align} \label{eq: general lb}
        R_T(\pi, \E) \;\in\; 
        \Omega \!\left( 
        \max_{\J \subseteq \pcal_k} (1 - c_{\J})^2 \sqrt{\frac{T |\J|}{\sum_{(p,s) \in \J} \wps}} \right),
    \end{align}
    where $c_{\J}$ denotes the maximum possible value of 
    $\tfrac{1}{|\J|}\sum_{(p,s)\in\J}\wps$ attainable by any policy.
    \end{lemma}
    
    \noindent
    In the following, we propose specific candidate subsets $\J$ that make the bound in Lemma~\ref{lem: general-lb} large while keeping $1 - c_{\J}$ constant. 
    
    Recall that $\wps$ is the fraction of time that policy $\pi$ observes $\mathbf{X}_p = s$ in the interaction with instance $\V$. Due to the construction of this instance, at each round $t$, the vector $\mathbf{x}^{(t)}$ can have at most $m$ (intervention size) entries different from $1$, because any non-intervened variable is always equal to $1$. Let 
    $$
    \B_m \coloneq  \Big\{ \mathbf{b} \in [\ell]^n \;\big|\; \sum_{i \in [n]} \mathbbm{1}_{\{b_i \neq 1\}} \leq m \Big\}
    $$ 
    denote the set of all possible realizations that can be observed by any algorithm interacting with $\V$.  
    
    Then for any set $\J$, we define 
    \begin{align} \label{def: m(j)}
        M(\J) \coloneq  \max_{\mathbf{b} \in \B_m} \sum_{(p,s) \in \J} \mathbbm{1}_{\{\mathbf{b}_p = s\}},
    \end{align}
    which represents the maximum number of pairs $(p,s) \in \J$ that can simultaneously be observed in a single round under any policy. By this definition, we have
    \begin{align} \label{eq: sum w J}
        \sum_{(p,s) \in \J} \wps 
        = \frac{1}{T} \sum_{t \in [T]} \sum_{(p,s) \in \J} \mathbb{E} \left[ \mathbbm{1}_{\{\mathbf{x}^{(t)}_p = s\}} \right] 
        \;\leq\; M(\J).
    \end{align}
    
    \medskip
    For any integer $i \in [k+1]$, define the set $\J_i \subseteq \pcal_k$ as
    \begin{align*}
        \J_i \coloneq  \Big\{ (p,s) \;\Big|\; p \in \tbinom{[n]}{k}, \; \sum_{j \in [k]} \mathbbm{1}_{\{s_j = 1\}} < i \Big\}.
    \end{align*}
    
    Thus $\J_i$ consists of pairs $(p,s)$ where $p$ is any $k$-subset of $[n]$ and $s$ contains fewer than $i$ entries equal to $1$. The following lemma provides the values of $|\J_i|$ and $M(\J_i)$ for any $i$. 

    \begin{lemma} \label{lem: values for J_i}
        For any $i \in [k]$, the set $\J_i$ has the following properties:
        \begin{align*}
            |\J_i| &= \binom{n}{k} \, \ell^k \, \pr \big(\text{B}(k, 1/\ell) < i\big), \\ 
            M(\J_i) &= \binom{n}{k} \, \pr \paran{\text{HG}(n,\, n - m,\, k) < i},
        \end{align*}
        where $\text{B}(k,1/\ell)$ denotes a Binomial random variable with parameters $(k,1/\ell)$, and $\text{HG}(n, n - m, k)$ denotes a Hypergeometric random variable obtained by drawing $k$ items without replacement from a population of size $n$ with $n - m$ successes.
    \end{lemma}
    \begin{proof}
        For $|\J_i|$, note that $p$ can be any $k$-subset, so there are $\binom{n}{k}$ possibilities for $p$. For each $p$, the number of valid vectors $s \in [\ell]^k$ is
        $$
        \ell^k \; \pr \big(\text{B}(k, 1/\ell) < i\big),
        $$
        where $\text{B}(k,1/\ell)$ denotes a Binomial random variable with parameters $(k,1/\ell)$. This holds because each entry of $s$ equals $1$ with probability $1/\ell$, and we require fewer than $i$ entries equal to $1$. Therefore,
        $$
        |\J_i| = \binom{n}{k} \, \ell^k \, \pr \big(\text{B}(k, 1/\ell) < i \big).
        $$
        
        We now calculate $M(\J_i)$. For any $i \in [k]$, our claim is that the maximizer of~\eqref{def: m(j)} corresponds to an action that sets exactly $m$ variables to values different from $1$.  
        
        Consider an action $a$ that sets $r$ variables to values $\neq 1$ and $m-r$ variables to $1$ (the environment sets other variables also equal to $1$). Then under this action, the number of pairs $(p,s)$ with fewer than $i$ ones in $s$ is
        $$
        \sum_{j=0}^{i-1} \binom{n-r}{j} \binom{r}{k-j}.
        $$
        
        This is because to form a pair $(p,s)$ observed under this action:
        (i) we must choose $j$ coordinates among the $n-r$ coordinates fixed to $1$ to appear as $1$ in $s$,  
        (ii) and simultaneously choose $k-j$ coordinates among the $r$ coordinates fixed to non-$1$ values to appear as non-$1$ in $s$.  
        Thus the total number of such pairs is given by the above sum.

        Finally, we show that the maximum of the above sum occurs at the maximum possible value of $r$, which is $m$. 
    
        \textbf{Claim.} For fixed integers $n,k,i$ with $0\le k\le n$ and $i\in\{1,\dots,k+1\}$, the function
        $$
        S(r) = \sum_{j=0}^{i-1} \binom{n-r}{j}\binom{r}{\,k-j\,}, \qquad r=0,1,\dots,n,
        $$
        is increasing in $r$.
        
        \textbf{Proof of claim.}
        Interpret the sum via a hypergeometric law. Consider a population of $n$ items with $r$ \emph{successes} and $(n-r)$ \emph{failures.} Draw $k$ items uniformly at random without replacement, and let $I_r$ denote the number of failures in the sample. Then
        $$
        \mathbb{P}(I_r=j) = \frac{\binom{n-r}{j}\binom{r}{\,k-j\,}}{\binom{n}{k}},\qquad j=0,1,\dots,k,
        $$
        so
        $$
        S(r) = \binom{n}{k}\,\mathbb{P}\!\left(I_r < i \right).
        $$
            
        We prove $\mathbb{P}(I_r < i)$ is increasing in $r$ by a coupling argument. Move from $r$ to $r+1$ by converting one failure in the population into a success (leaving all other items unchanged). Use the \emph{same} random $k$-subset of indices to form both samples. If the converted item is not selected, the sample composition is unchanged, hence $I_{r+1}=I_r$. If it \emph{is} selected, one failure becomes a success, hence $I_{r+1}=I_r-1$. In all cases
        $$
        I_{r+1}\;\le\;I_r \quad\text{almost surely.}
        $$
        Therefore, for every threshold $t$,
        $$
        \mathbb{P}(I_{r+1} < t)\;\ge\;\mathbb{P}(I_r < t).
        $$
        Taking $t=i$ and multiplying by $\binom{n}{k}$ gives $S(r+1)\ge S(r)$, i.e., $S(r)$ is increasing in $r$.
  
        Therefore,
        $$
        M(\J_i) = S(m) = \binom{n}{k} \, \pr \paran{\text{HG}(n,\, n-m,\, k) < i}.
        $$
    \end{proof}

    By Lemma~\ref{lem: values for J_i} and Equation~\eqref{eq: sum w J} we obtain
    $$
    \frac{|\J_i|}{\sum_{(p,s) \in \J_i} \wps} 
    \;\geq\; 
    \frac{\ell ^k \ \pr \big(\text{B}(k, 1/\ell) < i\big)} {\pr \paran{\text{HG}(n,\, n - m,\, k) < i} }.
    $$
    It can be shown that the right-hand side as a function of $i$ first decreases and, after some point, increases again; thus the maximizer 
    $$
    i^* \;=\; \text{argmax}_{i \in [k+1]} \ \frac{|\J_i|}{\sum_{(p,s) \in \J_i} \wps}
    $$
    lies in $\{1,\,k+1\}$. Hence, only these two candidates are sufficient to prove the lower bound.
    
    For $i = 1$, we have
    $$
    \frac{|\J_1|}{\sum_{(p,s) \in \J_1} \wps} 
    \;\geq\;  
    \frac{\ell ^k \ \pr \big(\text{B}(k, 1/\ell) < 1\big)} {\pr \paran{\text{HG}(n,\, n - m,\, k) < 1} }
    \;=\; 
    \frac{\ell ^k \paran{\tfrac{\ell-1}{\ell}}^k }{\binom{m}{k} / \binom{n}{k}}
    \;=\; 
    (\ell - 1)^k \, \frac{\binom{n}{k}}{\binom{m}{k}}.
    $$
    For $i = k+1$, we obtain
    $$
    \frac{|\J_{k+1}|}{\sum_{(p,s) \in \J_{k+1}} \wps} 
    \;\geq\;  
    \frac{\ell ^k \ \pr \big(\text{B}(k, 1/\ell) < k+1\big)} {\pr \paran{\text{HG}(n,\, n - m,\, k) < k+1} } 
    \;=\; 
    \frac{\ell ^k \cdot 1}{1} 
    \;=\; 
    \ell ^ k. 
    $$
    
    Hence, for $i^*$,
    $$
    \frac{|\J_{i^*}|}{\sum_{(p,s) \in \J_{i^*}} \wps} 
    \;\geq\; 
    \max \!\left\{ (\ell - 1)^k \, \frac{\binom{n}{k}}{\binom{m}{k}}, \ \ell^k \right\}.
    $$
    In particular, since $\ell \ge 2$ in our setting, $\ell^k \ge 2$, and therefore
    $$
    1 - c_{\J_{i^*}} 
    \;=\; 
    1 - \frac{\sum_{(p,s) \in \J_{i^*}}\wps}{|\J_{i^*}|} 
    \;=\; 
    1 - \frac{1}{\frac{|\J_{i^*}|}{\sum_{(p,s) \in \J_{i^*}}\wps}} 
    \;\geq\; 
    1 - \frac{1}{2} 
    \;=\; 
    \frac{1}{2}. 
    $$
    
    Substituting this result in the bound of Lemma~\ref{lem: general-lb} yields the desired regret lower bound for any policy $\pi \in \Pi(\{k, \G\})$:
    $$
    R_T \big( \pi, \E \big) 
    \;\in\; 
    \Omega \!\left( \sqrt{\,T \, \max \!\left( (\ell-1)^k \, \frac{\binom{n}{k}}{\binom{m}{k}}, \ \ell^k \right)} \right),
    $$
    which completes the proof.

\end{proof}

\mlesskLB*
\begin{proof}
    This theorem follows directly from the previous result.  
    For any $k > m$, consider instances in which $k - m$ of the reward’s parents are \emph{neutral}, that is, variables whose values do not affect the distribution of the reward (or whose effect is so small that it cannot be detected within $T$ samples).  
    In such cases, the effective number of true parents is $m$, and the instance is statistically indistinguishable from one with exactly $m$ parents.  
    Consequently, the worst-case regret of any algorithm on instances with $k$ reward parents must be at least as large as the regret for instances with $m$ parents.  
    Therefore, the lower bound stated in the theorem also holds for all $k > m$.
\end{proof}

\knownkUB*

\begin{proof}
    We first recall the standard UCB regret bound, which will be used in the analysis of both algorithms.  

    For any horizon $T$ and any bandit environment $\V$ with $N$ arms, UCB satisfies
    \begin{align} \label{eq: ucb bound}
        R_T(\pi_{\mathrm{UCB}}, \V) \;\leq\; \cucb \sqrt{TN\ln(T)},
    \end{align}
    where $\cucb>0$ is a universal constant. While this bound was originally proved for independent arms, it also holds when dependencies exist among arms. For the proof, see \cite{bandit-book1-lattimore2020bandit, bandit-book2-bubeck2012regret}.  

    To analyze Algorithm~1, define the event
    $$
    \E_0 = \{ \mu^*_{\A'} < \mu_{a^*} \},
    $$
    where $\A'$ is the subset of arms sampled in the algorithm, $\mu^*_{\A'}$ is the maximum mean reward among arms in $\A'$, and $a^*$ is the globally optimal arm.  


    By the definition of $\alpha_k$ as the fraction of optimal arms in $\A_m$, and the calculation of this value in \eqref{eq: alpha_k value}, the algorithm samples exactly $\frac{1}{\alpha_k}\ln \sqrt{T}$ random arms (or all arms in $\A_m$ if this number exceeds $|\A_m|$). By Lemma~\ref{lem: non-optimal prob}, the probability that none of them are optimal satisfies
    $$
    \pr(\E_0) \;\leq\; \frac{1}{\sqrt{T}}.
    $$

    Putting these together, for any instance $\V \in \E$, we can bound the regret as
    \begin{align*}
        R_T(\texttt{Alg.1}[k], \V) 
        &= \pr(\E_0) \ R_T(\texttt{Alg.1}[k], \V) \mid \E_0   + \pr(\E_0^c) \ R_T(\texttt{Alg.1}[k], \V) \mid \E_0^c   \\
        &\leq T \cdot \pr(\E_0) + R_T(\texttt{Alg.1}[k], \V) \mid \E_0^c \\
        &\leq \sqrt{T} + \cucb \sqrt{T \, |\A'| \ln(T)},
    \end{align*}
    where with a slight abuse of notation, we use $R_T(\texttt{Alg.1}[k], \V) \mid \E_0$ to denote the expected regret under the event $\E_0$.

    Finally, substituting the size of $\A'$, i.e.,  $n_0$ in Algorithm \ref{algo: known-k}, yields
    \begin{align*}
        R_T(\texttt{Alg.1}[k], \V) \;\leq\;
        \begin{cases}
            2 \cucb \sqrt{\tfrac12 \, T \, \ell^k \tfrac{\binom{n}{k}}{\binom{m}{k}} \, \ln^2(T)} 
            \;\in\; \tilde{\ocal}\!\left(\sqrt{T \, \ell^k \tfrac{\binom{n}{k}}{\binom{m}{k}}}\right), \quad & m \geq k, \\[1.2ex]
            2\cucb \sqrt{T \, \ell^m \binom{n}{m} \ln(T)} 
            \;\in\; \tilde{\ocal}\!\left(\sqrt{T \, \ell^m \binom{n}{m}}\right), & m < k,
        \end{cases}
    \end{align*}
    which matches the theorem statement and completes the proof.
\end{proof}

\subsection{Proofs of Section \ref{sec: unknown-k}}

\unknownkLB*

\begin{proof}
    Analogous to the proof of the lower-bound in the previous section, since the empty graph is a subgraph of any $\G$, it suffices to establish the result when $\G$ is empty; the general case then follows immediately.

    Fix values $k_1, k_2$ such that $k_1 < k_2 \leq m \leq n$.  
    Consider the instance $\V \in \E(n, \ell, k_1)$ defined as follows.  
    The graph $\G$ over the non-reward variables is the empty graph, and the parent set of the reward is $\pa{Y} = \{X_1, X_2, \ldots, X_{k_1}\}$.  
    Since $\G$ is empty, the variables $X_i$ are mutually independent.  
    Moreover, each variable $X_i$ satisfies $\pr(X_i = 1) = 1$, meaning that every variable equals $1$ unless it is intervened upon.  
    
    For any vector $\mathbf{b} \in [\ell]^{k_1}$ representing the values of the reward’s parents, the reward distribution is
    $$
    Y \mid \pa{Y} = \mathbf{b} \;\sim\; \N(\mu_{\mathbf{b}}, 1),
    $$
    where
    $$
    \mu_{\mathbf{b}} = 
    \begin{cases}
        \Delta, & \text{if } \mathbf{b} = (1,1,\ldots,1), \\[0.3em]
        0, & \text{otherwise},
    \end{cases}
    $$
    and $\Delta > 0$ is a constant that will be specified later.  
    
    \vspace{0.5em}
    \noindent
    To define the family of alternative instances, we define $\M$ to be  the set of all pairs $(p,s)$ satisfying:
    \begin{enumerate}[label=(\roman*)]
        \item $p \in \binom{[n]}{k_2}$,
        \item $[k_1] \subset p$,
        \item $s \in [\ell]^{k_2}$,
        \item $\sum_{i=1}^{k_1} \mathbbm{1}_{s_i = 1} < k_1$ (i.e., at least one of the first $k_1$ entries of $s$ is not $1$),
    \end{enumerate}
    Here, $p$ represents a set of $k_2$ indices containing the first $k_1$ parent indices, and $s$ assigns values to those indices such that at least one of the first $k_1$ entries is not $1$.
    
    
    \vspace{0.5em}
    \noindent
    For each $(p,s) \in \M$, define an alternative instance $\vps \in \E(n, \ell, k_2)$ as follows.  
    The graph and the distributions of all non-reward variables are identical to those in $\V$.  
    For any $\mathbf{b} \in [\ell]^{k_2}$ representing the values of the parents of $Y$, define
    $$
    Y \mid \pa{Y} = \mathbf{b} \;\sim\; \N(\mu_{\mathbf{b}}, 1),
    $$
    where
    $$
    \mu_{\mathbf{b}} = 
    \begin{cases}
        \Delta, & \text{if } b_1 = b_2 = \ldots = b_{k_1} = 1, \\[0.3em]
        2\Delta, & \text{if } \mathbf{b} = s, \\[0.3em]
        0, & \text{otherwise}.
    \end{cases}
    $$
    
    \vspace{0.5em}
    \noindent
    Note that $\vps$ differs from $\V$ only in the set of parents and only in the reward distribution corresponding to the  interventions that set $\mathbf{X}_p = s$.  
    Formally, for all other interventions, the joint distribution of all variables $(X_1, \ldots, X_n, Y)$ are identical, while for actions setting $\mathbf{X}_p = s$, they differ only in the reward distribution.
    
    Fix a policy $\pi$, and let $\pr_{\V}$ and $\pr_{p,s}$ denote the probability measures induced by $T$ rounds of interaction between $\pi$ and $\V$, $\vps$, respectively.  
    Define $\tps$ as the random variable counting the number of rounds $t$ during which $\mathbf{x}^{(t)}_p = s$, and let $\wps = \mathbb{E}_{\V}[\tps] / T$ denote its expectation fraction under $\pr_{\V}$.
    
    By the divergence decomposition lemma (Lemma~15.1 in \cite{bandit-book1-lattimore2020bandit}), we obtain
    $$
    \kl{\pr_{\V}, \pr_{p,s}} 
    \;=\; \mathbb{E}_{\V}[\tps] \, \kl{\N(0,1), \N(2\Delta,1)} 
    \;=\; 2 \, T \wps {\Delta^2},
    $$
    where we used $\kl{\N(0,1), \N(2\Delta,1)} = {2\Delta^2}$ in the last equality.

    Now, we define the event $E$ as 
    $$
    E = \left\{ \sum_{t \in [T]} \mathbbm{1}_{\{ \mathbf{x}^{(t)}_{[k_1]} = (1,1,\ldots,1) \}}  \geq \frac{T}{2} \right\}.
    $$
    which captures the case where, in at least half of the rounds, the first $k_1$ variables are all equal to $1$. Then
    \begin{align*}
        &R_T(\pi, \vps) | E \geq \frac{T \Delta}{2} \qquad \forall (p,s) \in \M, \\
        &R_T(\pi, \V) | E^c \geq \frac{T \Delta}{2},
    \end{align*}
    where with a slight abuse of notation, $R_T(\pi, \V) | E$ represents the expected regret given the event $E$. This implies that 
    \begin{align*}
        R_T(\pi, \vps) + R_T(\pi, \V) \geq \frac{T \Delta}{2} \paran{ \pr_{p,s} (E) + \pr_{\V} (E^c)}.
    \end{align*}
    Then, by Bretagnolle-Huber inequality \ref{lem:Bretagnolle}, we obtain
    $$
     \pr_{p,s} (E) + \pr_{\V} (E^c) \geq \frac{1}{2} \exp \paran{\kl{\pr_{\V}, \pr_{p,s}} } = \frac{1}{2} \exp \paran{-2 \, T \wps {\Delta^2}}.
    $$
 
    Now since $\E(n, \ell, k_1) \subseteq \E(n, \ell, k_2)$, we have for any policy $\pi$, $R_T(\pi, \E(n, \ell, k_1)) \leq R_T(\pi, \E(n, \ell, k_2))$, then
    \begin{align*}
        2 R_T(\pi, \E(n, \ell, k_2)) &\geq R_T(\pi, \E(n, \ell, k_1)) + R_T(\pi, \E(n, \ell, k_2)) \\
        &\geq R_T(\pi, \V) + R_T(\pi, \vps) \geq \frac{T \Delta}{4} \exp \paran{- 2 \, T \wps {\Delta^2}} \\
        & \Longrightarrow \ln \paran{8 R_T(\pi, \E(n, \ell, k_2))} - \ln \paran{T \Delta} \geq - 2 \, T \wps {\Delta^2} \\
        & \Longrightarrow 2 \, T \wps {\Delta^2} \geq \ln \paran{\frac{T \Delta}{8 R_T(\pi, \E(n, \ell, k_2))}}. 
    \end{align*}
    Then letting $\Delta = \frac{16R_T(\pi, \E(n, \ell, k_2))}{T} $, we obtain
    \begin{align} \label{eq: w_ps lower bound}
        &\frac{2 \wps R_T(\pi, \E(n, \ell, k_2))^2 }{T} \geq \ln(2), \nonumber \\
        &\Longrightarrow \wps \geq \ln(\sqrt{2}) \frac{T}{R_T(\pi, \E(n, \ell, k_2))^2}.
    \end{align}
    Thus we have derived a lower bound on $\wps$, the expected fraction of rounds (under $\V$) in which $\pi$ observes the low-reward configuration $\mathbf{x}_p = s$. 
    Note that without loss of generality, we can assume that $\Delta\leq1$. This is because when $\delta>1$, then the regret of $\pi$ exceeds $T/16$, which clearly satisfies the proposed lower bound in this theorem. 

    Now, to bound the regret using this lower bound, let $Q = \{2,3,\ldots, \ell\}^{k_1}$. For any $\qb \in Q$, let $w_{\qb}$ denote the expected fraction of rounds $t$ in which $\mathbf{x}^{(t)}_{[k_1]} = \qb$ during the interaction between $\pi$ and $\V$. In this case, the regret can be bounded as
    \begin{align} \label{eq: regret w_q bound}
        R_T(\pi, \V) \;\geq\; \Delta \sum_{\qb \in Q} T \, w_{\qb}. 
    \end{align}
    To bound $w_{\qb}$, for each $\qb \in Q$ define $\M_{\qb}$ as the set of pairs $(p,s) \in \M$ with $s_{[k_1]} = \qb$. Then
    $$
        |\M_{\qb}| \;=\; \binom{n - k_1}{\,k_2 - k_1\,} (\ell - 1)^{\,k_2 - k_1},
    $$
    because the indices $\{1,2,\ldots,k_1\}$ must be contained in $p$ and their values are fixed to $\qb$, leaving $\binom{n - k_1}{k_2 - k_1}$ choices for the remaining indices and $(\ell - 1)^{k_2-k_1}$ choices for their non-one values.  
    
    On the other hand, for any $\bb \in [\ell]^n$, the maximum number of pairs $(p,s) \in \M_{\qb}$ such that $\bb_p = s$ is $\binom{m - k_1}{k_2 - k_1}$. Indeed, under $\V$, any non-intervened variable equals $1$, while for $(p,s)\in \M$ all entries of $s$ beyond the first $k_1$ are non-one. Thus, the only way to realize $\bb_p = s$ is to (i) intervene to set the first $k_1$ variables to $\qb$ (none of which is $1$) and (ii) set $m-k_1$ additional variables to non-one values, yielding at most $\binom{m - k_1}{k_2 - k_1}$ distinct matches.
    
    Combining this counting argument with the lower bound in \eqref{eq: w_ps lower bound} gives, for each $\qb \in Q$,
    \begin{align} \label{eq: w_q lower bound}
        w_{\qb} 
        \;\geq\; 
        \frac{\sum_{(p,s) \in \M_{\qb}} w_{p,s}}{\max_{\bb \in [\ell]^n} \sum_{(p,s) \in \M_{\qb}} \mathbbm{1}_{\{ \bb_p = s \}}}
        \;\geq\;
        \frac{\tfrac{\ln(\sqrt2)\, T \, |\M_{\qb}|}{R_T(\pi, \E(n, \ell, k_2))^2}}{\binom{m - k_1}{k_2 - k_1}}
        \;\geq\;
        c \, \frac{T \, (\ell - 1)^{\,k_2 - k_1} \binom{n - k_1}{k_2 - k_1}}{R_T(\pi, \E(n, \ell, k_2))^2 \binom{m - k_1}{k_2 - k_1}},
    \end{align}
    for a universal constant $c>0$. Plugging \eqref{eq: w_q lower bound} into \eqref{eq: regret w_q bound} yields
    \begin{align} \label{eq: unknown first lower bound}
        R_T(\pi, \E(n, \ell, k_1)) 
        &\geq R_T(\pi, \V) 
        \;=\; T \Delta \sum_{\qb \in Q} w_{\qb} \nonumber \\
        &\geq 16c \, |Q| \, T \, \frac{R_T(\pi, \E(n, \ell, k_2))}{T}
        \, \frac{T \, (\ell - 1)^{\,k_2 - k_1} \binom{n - k_1}{k_2 - k_1}}{R_T(\pi, \E(n, \ell, k_2))^2 \binom{m - k_1}{k_2 - k_1}} \nonumber \\
        &\geq 16c \, \frac{T \, (\ell - 1)^{\,k_2} \binom{n - k_1}{k_2 - k_1}}{R_T(\pi, \E(n, \ell, k_2)) \binom{m - k_1}{k_2 - k_1}} \nonumber \\ 
        \Longrightarrow R_T(\pi, \E(n, \ell, k_1)) & R_T(\pi, \E(n, \ell, k_2)) \;\in\; \Omega\!\left( T \, (\ell - 1)^{\,k_2} \, \frac{\binom{n - k_1}{k_2 - k_1}}{\binom{m - k_1}{k_2 - k_1}} \right) 
    \end{align}
    In the above, we use the fact that $|Q|=(\ell-1)^{k_1}$. 
    This proves the first term inside the $\max$ in the theorem.

    We now show that the product is also lower bounded by the second term, namely $T \, \ell^{k_2}$.  
    For this, let 
    \[
    R = \{(p_0, s) \in \M \mid p_0 = [k_2] \} \subseteq \M.
    \]
    Then, each instance $\vpzs$ in $R$ corresponds to the setting where the reward’s parents are the first $k_2$ nodes, and the reward mean is $2\Delta$ whenever their values are equal to $s$.  
    Since $R \subseteq \M$, for each $(p_0, s) \in R$, the lower bound in \eqref{eq: w_ps lower bound} also holds:
    \[
    w_{p_0, s} \;\geq\; c \, \frac{T}{R_T(\pi, \E(n, \ell, k_2))^2}.
    \]
    
    Now, because each $\vpzs$ differs from $\V$ only when $\mathbf{x}_{[k_2]} = s$, and there are $|R| = (\ell^{k_1} - 1)\ell^{k_2 - k_1}$ such distinct configurations $s$ with $(p_0, s) \in \M$, the total expected fraction of rounds in which $\pi$ observes one of these configurations when interacting with $\V$ is at least
    \[
    \sum_{(p_0, s) \in R} w_{p_0, s} 
    \;\geq\; c \, |R| \, \frac{T}{R_T(\pi, \E(n, \ell, k_2))^2} \geq {c} \; \frac{\ell^{k_2}}{2} \; \frac{T}{R_T(\pi, \E(n, \ell, k_2))^2}
    \]
    Moreover, by construction of $\V$, every such configuration $\mathbf{x}_{[k_2]} = s$ with at least one non-one value in the first $k_1$ entries corresponds to a suboptimal mean reward. Thus, the regret on $\V$ can be written as
    \begin{align*}
        R_T(\pi, \V)
        &= \Delta \, T \sum_{(p_0, s) \in R} w_{p_0, s}
        \;\geq\;
        \frac{c}{2} \, \Delta \, T^2 \, \frac{\ell^{k_2}}{R_T(\pi, \E(n, \ell, k_2))^2}.
    \end{align*}
    
    Finally, substituting $\Delta = \frac{16R_T(\pi, \E(n, \ell, k_2))}{T}$ as before yields
    \begin{align*}
        R_T(\pi, \E(n, \ell, k_1)) \;\geq\; R_T(\pi, \V)
        \;\geq\; c' \, \frac{T \, \ell^{k_2}}{R_T(\pi, \E(n, \ell, k_2))},
    \end{align*}
    which implies
    \[
    R_T(\pi, \E(n, \ell, k_1)) \, R_T(\pi, \E(n, \ell, k_2))
    \;\in\;
    \Omega \big( T \, \ell^{k_2} \big).
    \]
    This establishes the second term in the $\max$ expression of the theorem and completes the proof.

\end{proof}

\unknownkUB*

\begin{proof}
    We first introduce a few notations. Recall the definition of $\alpha_k$ as the fraction of optimal arms in $\A_m$ when $|\pa{Y}| = k$, whose value is given in \eqref{eq: alpha_k value}. We define
    $$
    N_k \;=\; \frac{1}{\alpha_k} \ln \sqrt{T}.
    $$
    Fix a value of $k$ with $k \leq m$ and an arbitrary instance $\V \in \E(n, \ell, k)$.  
    Let $R_i$ denote the regret of Algorithm~\ref{algo: unknown-k} on $\V$ during phase $i$, $T_i = \sum_{j=1}^i \Delta T_j$, and let $r = \left\lceil \log_2 \sqrt T \right \rceil$.  

    Note that, since we introduce the mixture arms as new arms after phase one, the assumptions that (i) the reward distributions under different actions are independent and (ii) each is $1$-sub-Gaussian, no longer hold. Indeed, the reward of a mixture arm is a mixture of several sub-Gaussian variables corresponding to previously played arms, and these mixture components may overlap across different arms, introducing correlations. 

    To address this, we make two observations. 
    First, the regret upper bound of the UCB algorithm remains valid even when the arms are correlated. 
    Second, by Lemma~\ref{lem: sub-gaussian}, the reward distribution of any mixture arm is itself sub-Gaussian with parameter $\sigma^2 = \tfrac{5}{4}$. 
    This increased sub-Gaussian parameter scales the regret bound of UCB by a factor of $\tfrac{5}{4}$, which we absorb into the UCB constant $C_{\mathrm{UCB}}$ introduced in~\eqref{eq: ucb bound}.

    \noindent
    We first show that Algorithm~\ref{algo: unknown-k} is well-defined, i.e., it executes all $T$ rounds.  
    Since it runs $\Delta T_i$ rounds in phase $i$, the total number of rounds is
    \begin{align*}
        \sum_{i \in [i_f]} \Delta T_i 
        &= 2^{r} \ceillnm \sum_{i \in [i_f]} 2^i \\[0.3em]
        &= 2^{r} \ceillnm (2^{i_f + 1} - 2) \;\geq\; 2^{2i_f} \Big\lceil \frac{\ell n}{m} \Big\rceil 
        \;\geq\; T \frac{m}{\ell n} \Big\lceil \frac{\ell n}{m} \Big\rceil 
        \;\geq\; T,
    \end{align*}
    where we used $r \ge i_f$ and $i_f = \Big\lceil \log_2 \sqrt{T\frac{m}{\ell n}} \Big\rceil$.  
    
    \noindent
    For each phase $i$, let $\F_{i-1}$ denote the information available at the start of phase $i$ (i.e., all previous observations and the random arms selected in $S_i$).  
    We decompose the regret in phase $i$ into two components:
    \begin{align*}
        R_i = \Delta T_i \mu^* -\!\! \sum_{t = T_{i-1} + 1}^{T_i} \mathbb{E}[\mu_{a_t}] &= \Delta T_i (\mu^* - \mu^*_i) 
            \;+\; \Big(\Delta T_i \mu^*_i - \sum_{t = T_{i-1} + 1}^{T_i} \mathbb{E}[\mu_{a_t}] \Big)\\
            &= \rione + \ritwo,
    \end{align*}
    where $\mu^*_i = \max_{a \in S_i \cup M} \mu_a$ is the best mean reward among arms considered in phase $i$.  
    
    \vspace{0.5em}
    \noindent
    \textbf{Bounding $\ritwo$.}   
    Since each phase runs a standard UCB subroutine, the learning term $\ritwo$ is bounded by the UCB regret upper bound \eqref{eq: ucb bound}:
    \begin{align} \label{eq: r_i2 upper bound}
        \mathbb{E} \left[ \ritwo \right]
        &\leq \cucb \sqrt{\delti |S_i \cup M| \ln(\delti)} \nonumber\\
        &\leq \cucb \sqrt{2^{r + i} \ceillnm (2^{r - i + 1} + i -1) \ln(T)} \nonumber\\
        &\leq \cucb \sqrt{2^{2 r + 1}  \ceillnm \ln(T) \log_2(T)} \nonumber\\
        &\leq \cucb \sqrt{16 \, T \ceillnm \ln^2(T)} 
        \;\leq\; 4\sqrt{2}\,\cucb \sqrt{T \, \paran{\frac{\ell n }{m}} \, \ln^2(T)} 
        \;\in\; \tilde{\ocal} \paran{\sqrt{T \paran{\frac{\ell n }{m}}}},
    \end{align}
    where in the second inequality, we use the fact that $|S_i\cup M|\leq 2^{r-i+1}+(i-1)$ and in the third inequality, we used
    \begin{align*}
        \forall i \in [i_f]: \quad 
        2^{r + i} (2^{r - i + 1} + i -1)
        &< 2^{2 r + 1} + 2^{2 r}i_f 
        \;\leq\; 2^{2 r + 1} + 2^{2 r}\log_2(T) \\
        &\leq 2^{2 r + 1} \log_2(T),
    \end{align*}
    using $i_f = \Big\lceil \log_2 \sqrt{T\frac{m}{\ell n}} \Big\rceil \leq \log_2(T)$, which holds for $T > 3$.  
    The next inequality follows from
    \begin{align*}
        2^{2 r} &\leq 2^{2 (\log_2 \sqrt{T} + 1)} = 4T, \\
        \log_2(T) &\leq 2 \ln(T),
    \end{align*}
    and finally, we use $\lceil x \rceil \leq 2x$ for $x \geq 1$.  

    \vspace{0.5em}
    \noindent
    \textbf{Bounding $\rione$.}  
    To bound the term $\rione$, let
    $$
        i^* = \max \big\{ i \in [i_f] \;\big|\; q_i \geq N_k \big\}.
    $$
    If no such $i$ exists, we will have
    $$
        N_k = \frac{1}{\alpha_k} \ln \sqrt{T} = \frac{\ell^k \binom{n}{k}}{\binom{m}{k}} \ln \sqrt{T} > q_1 = 2^{r} \geq \sqrt{T},
    $$
    which implies
    $$
        \frac{\ell^{k} \binom{n}{k}}{\binom{m}{k}} \sqrt{T \frac{m}{\ell n}}
        > T \, \frac{\sqrt{\frac{m}{\ell n}}}{\ln \sqrt{T}}.
    $$
    In the above inequality, the left-hand side corresponds to the target regret upper bound, while the right-hand side grows linearly in $T$, ignoring the logarithmic factor. In this case, the bound trivially holds. Thus, for the remainder, we may assume that $i^*$ exists.
    
    For each $i$, define the event
    $$
        \E_i = \{ \text{no optimal arm is contained in } S_i \}.
    $$
    By Lemma~\ref{lem: non-optimal prob}, for any $i \leq i^*$, since $q_i \geq N_k$ (recall that $q_i$ is decreasing in $i$), we have $\pr(\E_i) \leq \tfrac{1}{\sqrt{T}}$.  
    Under the complement event $\E_i^c$, at least one optimal arm is included in $S_i$, implying that $\mu^*_i = \mu^*$ and therefore $\rione = 0$.  
    
    Consequently, the expected contribution of $\rione$ for phases $i \leq i^*$ can be bounded as
    \begin{align} \label{eq: rione bound small i}
        \mathbb{E} \left[ \rione \right] = \Delta T_i \, \mathbb{E}[\mu^* - \mu^*_i] \leq \Delta T_i \, (\mu^* - \mu_{\min}) \, \pr(\E_i) \leq \Delta T_i \, \pr(\E_i) 
        \;\leq\; \frac{\Delta T_i}{\sqrt{T}} 
        \;\leq\; \sqrt{T},
    \end{align}
    where $\mu_{\min}$ denotes the smallest mean reward among all arms, and the second inequality follows from the boundedness of rewards in $[0,1]$.
    
    For each $i > i^*$, note that the mixture arm $\tilde{a}_{i^*}$, constructed at the end of phase $i^*$, is included in the arm set for phase $i$. Hence,
    $$
        \mathbb{E}\!\left[ \rione \right]
        = \Delta T_i \, \mathbb{E}[\mu^* - \mu^*_i]
        \;\leq\; \Delta T_i \, \mathbb{E}[\mu^* - \mu_{\tilde{a}_{i^*}}]
        = \Delta T_i \, (\mu^* - \mathbb{E}[\mu_{\tilde{a}_{i^*}}])
        = \Delta T_i \, \frac{\mathbb{E}[R^{(2)}_{i^*}]}{\Delta T_{i^*}},
    $$
    where the last equality holds because the regret of the mixture arm equals the average regret of the actions played during phase $i^*$.
    
    By substituting the bound from~\eqref{eq: r_i2 upper bound} and using the fact that $\Delta T_i\leq T$, we obtain
    \begin{align*}
        \mathbb{E}\!\left[ \rione \right]
        &\leq \Delta T_i \, \frac{4 \, \cucb \sqrt{T \ceillnm \ln^2(T)}}{\ceillnm 2^{r+i^*}}
        \;\leq\; 4 \, T \, \cucb \sqrt{ \frac{T \ln^2(T)}{\ceillnm 2^{2r + 2i^*}}}
        \;\leq\; 4 \, T \, \cucb \sqrt{ \frac{\ln^2(T)}{\ceillnm 2^{2i^*}} },
    \end{align*}
    where the last inequality uses $r \geq \log_2(\sqrt{T})$, implying $2^{2r} \geq T$.
    
    From the definition of $i^*$, we have
    $$
       N_k > q_{i^*+1}=2^{r - i^*}\quad \& \quad q_{i^*} = 2^{r - i^* + 1} \geq N_k 
        \quad\Longrightarrow\quad 
        2^{i^*} N_k > 2^r \geq \sqrt{T},
    $$
    which gives
    $$
        2^{2i^*} > \frac{T}{N_k^2} = \frac{T \alpha_k^2}{\ln(T)}.
    $$
    Substituting this into the previous bound yields
    \begin{align} \label{eq: rione bound large i}
        \mathbb{E}\!\left[ \rione \right]
        \leq 4 \, T \, \cucb \sqrt{\frac{\ln^3(T)}{\ceillnm T \alpha_k^2}}
        = 4 \, \cucb \, \frac{1}{\alpha_k} \sqrt{\frac{T \ln^3(T)}{\ceillnm}}
        \in \tilde{\mathcal{O}}\!\left(
            \sqrt{T \, \frac{m}{n}} \;
            \ell^{k - \frac{1}{2}} \,
            \frac{\binom{n}{k}}{\binom{m}{k}}
        \right).
    \end{align}
    
    Finally, combining~\eqref{eq: r_i2 upper bound}, \eqref{eq: rione bound small i}, and~\eqref{eq: rione bound large i}, we obtain
    $$
        \forall i \in [i_f]:
        \quad
        \mathbb{E}[R_i]
        \in \tilde{\mathcal{O}}\!\left(
            \sqrt{T \, \frac{m}{n}} \;
            \ell^{k - \frac{1}{2}} \,
            \frac{\binom{n}{k}}{\binom{m}{k}}
        \right).
    $$
    Since the number of phases $i_f$ is logarithmic in problem parameters, the same upper bound holds for the total regret (up to logarithmic factors hidden in the $\tilde{\mathcal{O}}$ notation).
    
    This completes the proof for the case $k \leq m$.  
    For $k > m$, all the proof steps hold by setting $k = m$, resulting in the regret bound presented in the theorem. 

    
\end{proof}

\pareto*

\begin{proof}
    First, consider the case $m = n$.  
    In this setting, ignoring logarithmic factors, the regret vector of Algorithm~\ref{algo: unknown-k} is
    \[
        \Rb = \left[ \sqrt{T} \, \ell^{\frac{1}{2}}, \; \sqrt{T} \, \ell^{\frac{3}{2}}, \; \ldots, \; \sqrt{T} \, \ell^{m - \frac{1}{2}} \right].
    \]
    For the sake of contradiction, assume there exists a policy $\pi$ with regret vector $\Rb'$ that rate-Pareto dominates $\Rb$.  
    
    Note that for $k = 1$, the lower bound in Theorem~\ref{thm: m-geq-k-lower-bound} establishes that the optimal rate is $\Theta(\sqrt{T \ell})$, which coincides with the first entry of $\Rb$.  
    Hence, $\Rb'$ cannot achieve a strictly better rate for $k = 1$, implying that $R'_1 = \Theta(R_1)$.  
    This means there exists a $k' > 1$ and a constant $C$ such that
    \[
        R'_{k'} \leq C R_{k'} = C \sqrt{T} \, \ell^{k' - \frac{1}{2}},
    \]
    and the inequality does not hold in the reverse direction (i.e., $\Rb'$ improves over $\Rb$ at $k'$).  
    However, by Theorem~\ref{thm: unknown-k-lower-bound}, there exists a constant $C'$ such that
    \[
        R'_1 \, R'_{k'} \;\geq\; C' T \, \ell^{k'} 
        \quad\Longrightarrow\quad
        R'_{k'} \;\geq\; C'' \sqrt{T} \, \ell^{k' - \frac{1}{2}},
    \]
    for some universal constant $C''$.  
    This shows that $R_{k'}$ and $R'_{k'}$ have the same rate, which contradicts the assumption that $\Rb'$ rate-Pareto dominates $\Rb$.  
    Therefore, Algorithm~\ref{algo: unknown-k} is rate-Pareto optimal in this setting.
    
    \vspace{0.5em}
    \noindent
    Now consider the general case.  
    If $\ell \in \Omega(m)$, then $\frac{\ell^m}{(\ell - 1)^m} \in \mathcal{O}(1)$.  
    Ignoring logarithmic and constant factors, the regret vector $\Rb$ from the theorem satisfies:
    \[
        R_1 = \sqrt{T \frac{(\ell - 1)n}{m}}, 
        \qquad 
        R_k = \sqrt{T \frac{m}{n}} \, \frac{m}{n} (\ell - 1)^{k - \frac{1}{2}} \frac{\binom{n}{k}}{\binom{m}{k}}, \quad \text{for } k > 1.
    \]
    
    Suppose again that there exists a policy $\pi$ with regret vector $\Rb'$ that rate-Pareto dominates $\Rb$.  
    Since Algorithm~\ref{algo: unknown-k} is optimal for $k = 1$, we have $R'_1 \in \mathcal{O}(R_1)$.  
    Therefore, there must exist a $k' > 1$ and a constant $C$ such that $R'_{k'} \leq C R_{k'}$.  
    By Theorem~\ref{thm: unknown-k-lower-bound}, there exists a constant $C'$ satisfying
    \[
        R'_1 R'_{k'} \;\geq\; C' T (\ell - 1)^{k'} \frac{\binom{n - 1}{k' - 1}}{\binom{m - 1}{k' - 1}}
        \quad\Longrightarrow\quad
        R'_{k'} \;\geq\; C'' \sqrt{T \frac{m}{n}} \, \frac{m}{n} (\ell - 1)^{k' - \frac{1}{2}} \frac{\binom{n}{k'}}{\binom{m}{k'}},
    \]
    where $C''$ is a universal constant and we used 
    \[
        \frac{\binom{n - 1}{k' - 1}}{\binom{m - 1}{k' - 1}} 
        = \frac{\binom{n}{k'}}{\binom{m}{k'}} \frac{m}{n}.
    \]
    Since this lower bound matches the rate of $R_{k'}$, we again reach a contradiction, proving that regret vector $\Rb$ is rate-Pareto optimal.
    
\end{proof}

\section{ADDITIONAL EXPERIMENTS} \label{sec: additional experiments}

This section presents further details on the experiments in the main text and more experimental results. 

This section provides additional details about the experimental setup described in the main text and reports further empirical results.

\begin{figure}[t]
    \centering
    \includegraphics[width=0.7\linewidth]{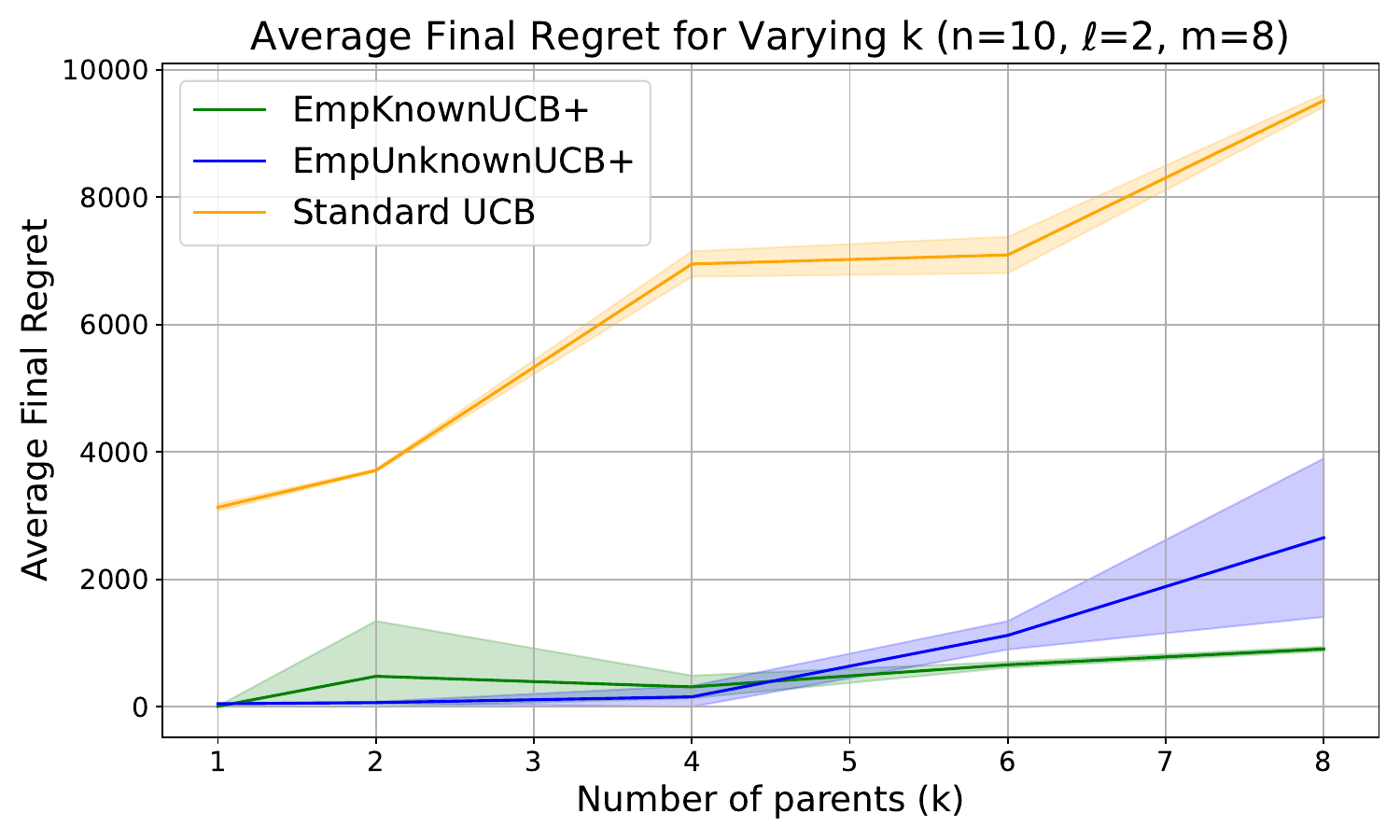}
    \caption{Average cumulative regret of algorithms at time $T$ for varying numbers of reward parents $k \in {1,2,4,6,8}$ on random instances with $n = 10$, $\ell = 2$, $m = 8$, and $T = 30{,}000$. As $k$ increases, the environment becomes more complex, leading to higher regret. Our proposed algorithms consistently achieve lower regret across all settings.}
    \label{fig: vary k}
\end{figure}

\textbf{Instance Generation.}  
For each instance $\V \in \E(n, \ell, k)$ used in the experiments, the causal graph $\G$ is generated as an Erdős–Rényi random graph with edge probability $p = \tfrac{2}{n}$.  
The reward’s parent set is selected uniformly at random from all subsets of size $k$.  

For every variable $X \in \X$ with parent set $Pa_{\G}(X)$, and any $\mathbf{z} \in [\ell]^{| Pa_{\G}(X) |}$, the conditional distribution $\pr(X \mid Pa_{\G}(X) = \mathbf{z})$ is modeled as a categorical distribution with probability vector $\mathbf{v}_{X}(\mathbf{z}) \in \Delta^{\ell - 1}$, constructed as follows.  
First, we sample a base vector $\mathbf{v}_{X}$ from a $\text{Dirichlet}(1,1,\ldots,1)$ distribution, which defines a random categorical distribution shared across all parent configurations $\mathbf{z}$.  
Then, for each $\mathbf{z}$, we draw another random vector $\mathbf{u}$ from the same Dirichlet distribution and define
\[
    \mathbf{v}_{X}(\mathbf{z}) = (1 - \beta)\,\mathbf{v}_{X} + \beta\,\mathbf{u},
\]
where $\beta$ is the \emph{parent-effect} parameter controlling how strongly the parents influence $X$.  
For example, $\beta = 0$ corresponds to a node whose distribution is completely independent of its parents.  
In all experiments, we set $\beta = 0.7$ to allow moderate parent influence while maintaining stochasticity.  

The reward variable is binary, and for each combination of its parents’ values, the reward mean is drawn independently and uniformly from $[0,1]$.

\textbf{RAPS.}  
The RAPS algorithm includes a structural discovery subroutine that, when applied to any node, intervenes on that node across all $\ell$ possible values, performing an equal number of interventions per value, and measures the resulting changes in the distributions of other variables to identify its descendants.  
In our experiments, the number of interventions per value was set to $\epsilon \log(10)$, which corresponds to the number of samples required to detect a change of at least $\epsilon$ with probability $0.1$.  
A node was identified as a descendant if the probability of at least one of its values changed by more than $\epsilon$ under these interventions. The $\epsilon$ is set to $0.05$. 

However, this exploration phase requires a large number
of rounds and is repeated once per parent, since the discovery procedure must be repeated for each parent.  
Consequently, we excluded RAPS from experiments involving more than one parent ($k > 1$), as its cumulative regret becomes prohibitively large in such settings.

\textbf{Effect of $k$.}  
To examine the impact of the number of reward parents on algorithm performance, we evaluate all methods on a set of randomly generated instances with parameters $n = 10$, $\ell = 2$, and $m = 8$, while varying $k \in \{1, 2, 4, 6, 8\}$.  
Figure~\ref{fig: vary k} reports the final regret for each algorithm after $T = 30{,}000$ rounds.  
As expected, the regret generally increases with $k$, reflecting the greater structural complexity and larger effective action space.  

\textbf{Effect of $m$.}  
To analyze the influence of the intervention size, we run all algorithms on random instances with parameters $n = 8$, $\ell = 2$, and $k = 2$, varying $m \in \{2, 3, 4, 5, 6, 7, 8\}$.  
Figure~\ref{fig: vary m} shows the final regret across algorithms after $T = 30{,}000$ rounds.  
Larger $m$ values correspond to broader interventions, allowing more informative exploration and hence lower regret, consistent with our theoretical predictions.

\begin{figure}[t]
    \centering
    \includegraphics[width=0.7\linewidth]{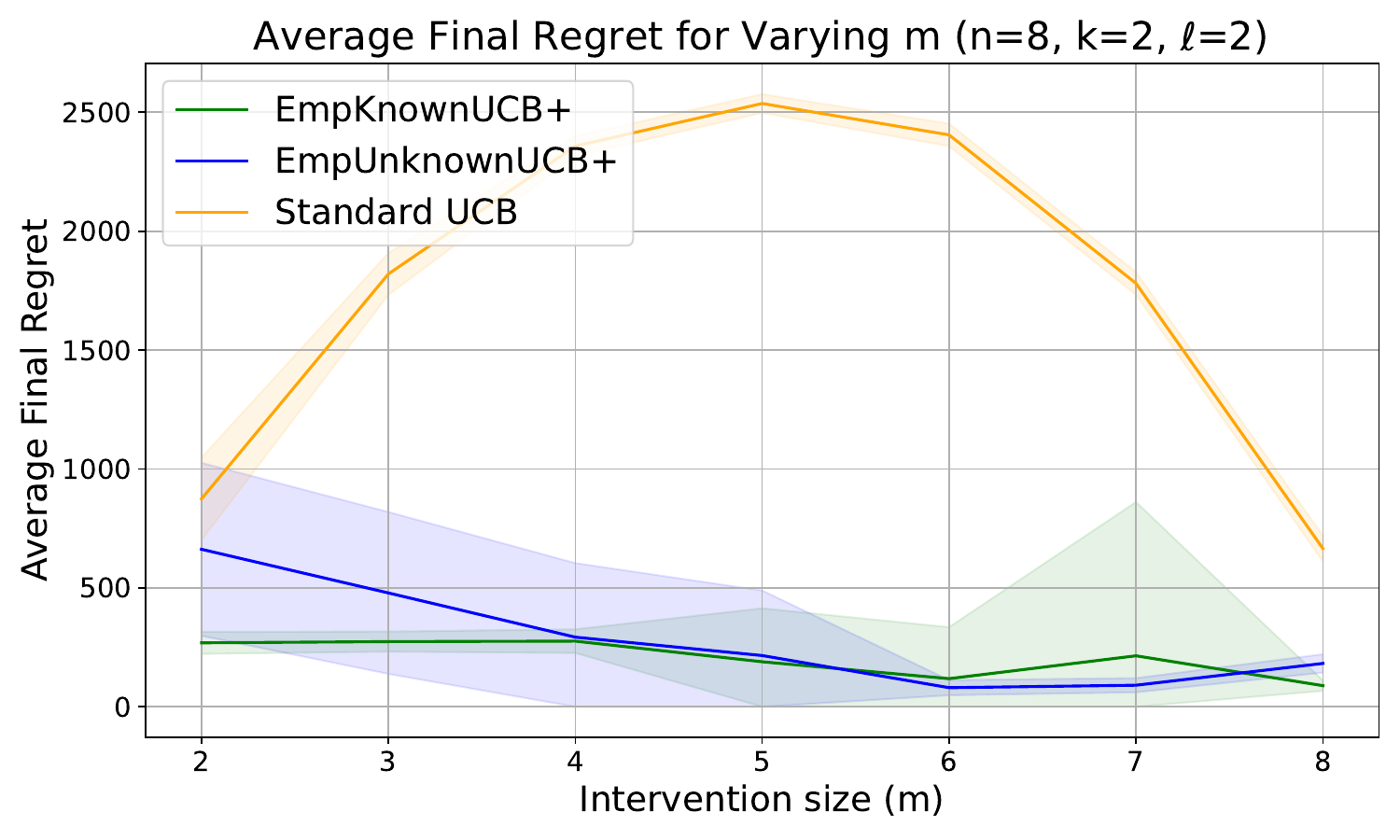}
    \caption{Average cumulative regret of algorithms at time $T$ for varying intervention sizes $m \in {2,3,4,5,6,7,8}$ on random instances with $n = 8$, $\ell = 2$, and $k = 2$. Larger intervention sizes enable more informative exploration, resulting in lower regret.}
    \label{fig: vary m}
\end{figure}

\end{document}